\PassOptionsToPackage{hyphens}{url}
\documentclass[onefignum,onetabnum]{siamonline171218}
\usepackage{amsfonts}
\usepackage{xfrac}
\usepackage{bm}
\usepackage{graphicx}
\usepackage{epstopdf}
\usepackage{algorithmic}
\usepackage{amssymb}
\usepackage{booktabs}
\usepackage{algorithmic}
\Crefname{ALC@unique}{Line}{Lines}
\usepackage[T1]{fontenc}
\usepackage{tikz}
\usepackage[caption=false,font=footnotesize]{subfig}
\usepackage{pgfplots}
\pgfplotsset{compat=newest}
\pgfplotsset{plot coordinates/math parser=false}
\usepackage{url}
\usepackage{color}
\usepackage{adjustbox}

\usepackage{todonotes}

\usetikzlibrary{plotmarks}

\newlength\figureheight
\newlength\figurewidth 

\ifpdf
  \DeclareGraphicsExtensions{.eps,.pdf,.png,.jpg}
\else
  \DeclareGraphicsExtensions{.eps}
\fi
\usepackage{amsopn}
\DeclareMathOperator{\diag}{diag}

\def\norm#1{\|#1\|} 

\definecolor{darkgreen}{rgb}{0.0,0.5,0.0}
\definecolor{amber}{rgb}{1.0, 0.75, 0.0}

\def\abs#1{\left|#1\right|} %
\def\norm#1{\left\|#1\right\|} %
\newcommand{\R}{\ensuremath{\mathbb{R}}}

\newcommand{\Aa}{\mathbf{A}}
\newcommand{\Tt}{\mathbf{T}}
\newcommand{\Vv}{\mathbf{V}}
\newcommand{\Ee}{\mathbf{E}}
\newcommand{\Uu}{\mathbf{U}}
\newcommand{\Ll}{\mathbf{L}}
\newcommand{\Qq}{\mathbf{Q}}
\newcommand{\Dd}{\mathbf{D}}
\newcommand{\Ww}{\mathbf{W}}

\newcommand{\rb}{\mathbf{r}}

\newcommand{\yb}{\mathbf{{{y}}}}
\newcommand{\qb}{\mathbf{q}}

\newcommand{\eps}{\varepsilon}
\newcommand{\xx}{{\bf x}}
\newcommand{\vb}{{\bf v}}

\newcommand{\bu}{{\bf u}}
\newcommand{\bv}{{\bf v}}

\def\abs#1{\left|#1\right|} %
\def\norm#1{\left\|#1\right\|} %

\DeclareMathOperator*{\argmin}{arg\,min}

\def\abs#1{\left|#1\right|} 
\def\norm#1{\left\|#1\right\|} 
\numberwithin{theorem}{section}

\newcommand{\TheTitle}{NFFT meets Krylov methods: Fast matrix-vector products for the graph Laplacian of fully connected networks}
\newcommand{\ShortTitle}{NFFT meets Krylov methods}
\newcommand{\TheAuthors}{Dominik Alfke, Daniel Potts, Martin Stoll, Toni Volkmer}

\headers{\ShortTitle}{\TheAuthors}
\title{{\TheTitle} 
}

\author{Dominik Alfke\thanks{Technische Universit\"at Chemnitz, Faculty of Mathematics, Chair of Scientific Computing, 09107 Chemnitz, Germany, ({\tt dominik.alfke@mathematik.tu-chemnitz.de})} \and Daniel Potts\thanks{Technische Universit\"at Chemnitz, Faculty of Mathematics, Chair of Applied Functional Analysis, 09107 Chemnitz, Germany, ({\tt daniel.potts@mathematik.tu-chemnitz.de})}\and Martin Stoll\thanks{Technische Universit\"at Chemnitz, Faculty of Mathematics, Chair of Scientific Computing, 09107 Chemnitz, Germany, ({\tt martin.stoll@mathematik.tu-chemnitz.de})} \and Toni Volkmer\thanks{Technische Universit\"at Chemnitz, Faculty of Mathematics, Chair of Applied Analysis, 09107 Chemnitz, Germany, ({\tt toni.volkmer@mathematik.tu-chemnitz.de})}
}

\usepackage{amsopn}
\ifpdf
\hypersetup{pdftitle={\TheTitle},
  pdfauthor={\TheAuthors}
}
\fi

\begin{document}

\maketitle

\begin{abstract}
The graph Laplacian is a standard tool in data science, machine learning, and image processing. The corresponding matrix inherits the complex structure of the underlying network and is in certain applications densely populated. This makes computations, in particular matrix-vector products, with the graph Laplacian a hard task. A typical application is the computation of a number of its eigenvalues and eigenvectors. Standard methods become infeasible as the number of nodes in the graph is too large. We propose the use of the fast summation based on the nonequispaced fast Fourier transform (NFFT) to perform the dense matrix-vector product with the graph Laplacian fast without ever forming the whole matrix. The enormous flexibility of the NFFT algorithm allows us to embed the accelerated multiplication into Lanczos-based eigenvalues routines or iterative linear system solvers \textcolor{black}{and even consider other than the standard Gaussian kernels}.
We illustrate the feasibility of our approach on a number of test problems from image segmentation to semi-supervised learning based on graph-based PDEs.
In particular, we compare our approach with the Nystr\"om method. Moreover, we present and test an enhanced, hybrid version of the Nystr\"om method, which internally uses the NFFT.
\end{abstract}


\pagestyle{myheadings}
\thispagestyle{plain}
\markboth{}{}

\section{Introduction}

Graphs are a fundamental tool in the modeling of imaging and data science applications \cite{VLu07,shuman2013emerging,belkin2002laplacian,belkin2003laplacian,henaff2015deep}.  To apply graph-based techniques, individual data points in a data set or pixels of an image represent the vertex set or nodes $V$ of the graph, and the edges indicate the relationship between the vertices. In a number of real-world examples, the graph is sparse in the sense that each vertex is only connected to a small number of other vertices, i.e., the graph affinity matrix is sparsely populated. In other applications, such as the mentioned data points or image pixels, the natural choice for the graph would be a fully connected graph, which is then reflected in dense matrices that represent the graph information. Naturally, if there is no underlying graph the most natural choice is the fully connected graph. As the eigenvectors of the corresponding graph Laplacian are crucial in reducing the complexity of the underlying problem or for the extraction of quantities of interest \cite{BerEG07,BerF12,shuman2013emerging}, it is important to compute them accurately and fast. If this matrix is sparse, numerical analysis has provided efficient tools based on the Lanczos process with sparse matrix-vector products that can compute the eigeninformation efficiently. For complex interactions leading to dense matrices, these methods suffer from the high cost of the matrix-vector product.

Our goal is hence to obtain the eigeninformation despite the fact that the graph is fully connected and without any a priori reduction of the graph information. For this we rely on a Lanczos procedure %
based on \cite{BaRe05}. This method needs
to perform the matrix-vector product in a fast way and thus, evaluating all information, without ever fully assembling the graph matrices. In a similar fashion the authors in \cite{BerF12} utilize the well-known Nystr\"om method to only work with partial information from the graph and only approximately represent the remaining parts. Such methods are well-known within the  fast solution of integral equations and have found applicability within the data science community \cite{drineas2005nystrom,liberty2007randomized}. The technique we present here is known as a fast summation method \cite{post02, postni04} and is based on the nonequispaced fast Fourier transform (NFFT), see \cite{KeKuPo09} and the references therein. We apply this method in the setting where the weights of the edges between the vertices are modelled by a Gaussian kernel function of medium to large scaling parameter, such that the Gaussian is not well-localized and most vertices interact with each other.
For the case of a smaller scaling parameter and consequently a more localized Gaussian, we refer to \cite{morariu08figtree}, \textcolor{black}{which is partially based on a technique presented \cite{yang2003improved} for Gaussian kernels.} Moreover, we remark that the NFFT-based fast summation method considered in this paper does not only support Gaussians but can handle various other rotational invariant functions.

The remaining parts of this paper are structured as follows. In Section~\ref{sec::graph_laplacian}, we first introduce the graph Laplacian and discuss the matrix structure.
In Section~\ref{sec::nfft}, we introduce the NFFT-based fast summation, which allows for computing fast matrix-vector products with the graph Laplacian.
In Section~\ref{sec::lanczos}, we then recall Krylov subspace methods and in particular the Lanczos method, which sits at the engine room of the numerical computations to obtain a small number of eigenvectors. We then show that the graph Laplacian provides the ideal environment to be used together with the NFFT-based fast summation,
and we obtain the NFFT-based Lanczos method.
In Section~\ref{sec::nystrom} we briefly discuss the Nystr\"om method as a direct competitor to our approach. We improve and accelerate this method, creating a new hybrid Nystr\"om-Gaussian-NFFT version, which incorporates the NFFT-based fast summation.
In Section~\ref{sec::res}, we present comparisons between the NFFT-based Lanczos method, the Nystr\"om method and the hybrid Nystr\"om-Gaussian-NFFT method with the direct application of the Lanczos method for a dense, large-scale problem. Additionally, we illustrate on a number of exemplary applications, such as spectral clustering and semi-supervised learning, that our approach provides a convenient infrastructure to be used within many different schemes.

\section{The graph Laplacian and fully connected graphs}
\label{sec::graph_laplacian}
We consider an undirected graph $G=(V,E)$ with the vertex set $V=\left\lbrace v_j\right\rbrace_{j=1}^{n}$ and the edge set $E$,  cf.~\cite{Chu97} for more information. An edge $e\in E$ is a pair of nodes $(v_j,v_i)$ with $v_j\neq v_i$ and $v_j,v_i\in V$. For weighted undirected graphs, such as the ones considered in this paper, we also have a weight function $w:V\times V\rightarrow\mathbb{R}$ with $w(v_j,v_i)=w(v_i,v_j)\textrm{ for all } j,i$. We assume further that the function is positive for existing edges and zero otherwise. The degree of the vertex $v_j\in V$ is defined as 
\begin{equation*}
d(v_j)=\sum_{v_i\in V}w(v_j,v_i).
\end{equation*}
Let $\Ww,\; \Dd \in \R^{n\times n}$ be the weight matrix and the diagonal degree matrix with entries $W_{ji} = w(v_j,v_i)$ and $D_{jj} = d(v_j)$.
Since we do not permit graphs with loops, $\Ww$ is zero on the diagonal.
Now the crucial tool for further investigations is the graph Laplacian $\Ll$ defined via 
\begin{equation*}
\Ll(v_j,v_i)=
\begin{cases}
d(v_j)&\textnormal{ if }v_j=v_i  \\ 
-w(v_j,v_i) & \textnormal{ otherwise, }
\end{cases}
\end{equation*}
i.e. $\Ll = \Dd - \Ww$.
The matrix $\Ll$ is typically known as the \textit{combinatorial graph Laplacian} and we refer to  \cite{VLu07} for an excellent discussion of its properties. Typically its normalized form is employed for segmentation purposes and we obtain the normalized Laplacian as
\begin{equation}\label{eq:nomlap}
\Ll_s=\Dd^{-1/2}\Ll\Dd^{-1/2}=\mathbf{I}-\Dd^{-1/2}\Ww\Dd^{-1/2},
\end{equation}
obviously a symmetric matrix. Another normalized Laplacian of nonsymmetric form is given by
\begin{equation*}
\Ll_w=\Dd^{-1}\Ll=\mathbf{I}-\Dd^{-1}\Ww.
\end{equation*}
For the purpose of this paper we focus on the symmetric normalized Laplacian $\Ll_s$ but everything we derive here can equally be applied to the nonsymmetric version, where we would then have to resort to nonsymmetric Krylov methods such as {\sc GMRES} \cite{gmres}. %
It is well known in the area of data science, data mining, image processing and so on that the smallest eigenvalues and its associated eigenvectors possess crucial information about the structure of the data and/or image \cite{VLu07,ZelP04, BerF12}. For this we state an amazing property of the graph Laplacian $\Ll$ for a general vector $\mathbf{u}\in\R^{n}$ with $n$ the dimension of $\Ll$
\begin{equation*}
 \mathbf{u}^{T}\Ll\mathbf{u}=\frac{1}{2}\sum_{j,i=1}^{n}W_{ji}(\mathbf{u}_j-\mathbf{u}_i)^{2},
\end{equation*}
which, as was illustrated in \cite{VLu07}, is equivalent to the objective function of the graph \textit{RatioCut} problem. Intuitively, assuming the vector $\mathbf{u}$ to be equal to a constant on one part of the graph $A$ and a different constant on the remaining vertices $\bar{A}$. In this case $\mathbf{u}^{T}\Ll\mathbf{u}$ only contains terms from the edges with vertices in both   $A$ and $\bar{A}$. Thus a minimization of $\mathbf{u}^{T}\Ll\mathbf{u}$ results in a minimal cut with respect to the edge weights across $A$ and $\bar{A}$. Obviously, $0$ is an eigenvalue of $\Ll$ and its normalized variants as $\Ll\mathbf{1}=\Dd\mathbf{1}-\Ww\mathbf{1}=\mathbf{0}$ by the definitions of $\Dd$ and $\Ww$ with $\mathbf{1}$ being the vector of all ones. Additionally, spectral clustering techniques heavily rely on the computation of the smallest $k$ eigenvectors \cite{VLu07} and recently semi-supervised learning based on PDEs on graphs introduced by Bertozzi and Flenner \cite{BerF12} utilizes a small number of such eigenvectors for a complexity reduction. It is therefore imperative to obtain efficient techniques to compute the eigenvalues and eigenvectors fast and accurately. 
Since we are interested in the $k$ smallest eigenvalues of the matrix $\Ll_s=\mathbf{I}-\Dd^{-1/2}\Ww\Dd^{-1/2}$ it is clear that we can compute the $k$ largest postive eigenvalues of the matrix $\Aa:=\Dd^{-1/2}\Ww\Dd^{-1/2}$.
In case that the graph $G=(V,E)$ is sparse in the sense that every vertex is only connected to a small number of other vertices and thus the matrix $\Ww$ is sparse, we can utilize the whole arsenal of numerical algorithms for the computations of a small number of eigenvalues, namely the Lanczos process \cite{book::golubvanloan}, the Krylov-Schur method \cite{stewartkrylov}, or the Jacobi-Davidson algorithm \cite{sleijpen2000jacobi}. In particular, the ARPACK library \cite{lehoucq1998arpack} in \textsc{Matlab} via the \texttt{eigs} function is a recommended choice. So frankly speaking, in the case of a sparse and symmetric matrix $\Ww$ the eigenvalue problem is fast and the algorithms are very mature. Hence, we focus on the case of fully connected graphs meaning that the matrix $\Ww$ is considered dense. 

The standard scenario for this case is that each node $v_j \in V$ corresponds to a data vector $\vb_j \in \R^d$ and the weight matrix is constructed as
\begin{equation}\label{eq::W_gaussian}
W_{ji} = 
w(v_j,v_i)=
\begin{cases}
\exp(-\norm{\vb_j-\vb_i}^2/\sigma^2)&\textnormal{ if } 
j \neq i, \\
0 & \textnormal{ otherwise}
\end{cases}
\end{equation}
with a scaling parameter $\sigma$.
For example, approaches with this kind of graph Laplacian have become increasingly popular in image processing \cite{romano2017little}, where the data vectors $\vb_j$ encode color information of image pixels via their color channels. The data point dimension may then be $d=1$ for grayscale images and $d=3$ for RGB images. Other applications may involve simple Cartesian coordinates for $\vb_j$. While Equation \eqref{eq::W_gaussian} is derived from a Gaussian kernel function $K(\mathbf{y}):=\exp(-\norm{\mathbf{y}}^2/\sigma^2)$, other applications might call for different kernel functions like
the ``Laplacian RBF kernel'' $K(\mathbf{y}) := \exp (-\norm{\mathbf{y}}/\sigma)$,
the multiquadric kernel $K(\mathbf{y}):=(\norm{\mathbf{y}}^2+c^2)^{1/2}$,
or inverse multiquadric kernel $K(\mathbf{y}):=(\norm{\mathbf{y}}^2+c^2)^{-1/2}$ for a parameter $c>0$, e.g.\ cf.\ Section~\ref{sec::kernel_ridge_regression}.
This means, the weight matrix may be of the form
\begin{equation}\label{eq::W_general}
W_{ji} = 
\begin{cases}
K(\vb_j-\vb_i)&\textnormal{ if } 
j \neq i, \\
0 & \textnormal{ otherwise.}
\end{cases}
\end{equation}

Often certain techniques are used to sparsify the Laplacian or otherwise reduce its complexity in order to apply the methods named above. In particular, sparsification has been proposed for the construction of preconditioners \cite{spielman2004nearly} for iterative solvers, which still require the efficient implementation of the matrix vector products. In image processing, this can be achieved by considering only patches or other reduced representations of the image \cite{ZelP04}. However, this might drop crucial nonlocal information encoded in the full graph Laplacian \cite{romano2017little,gilboa2008nonlocal}, which is why we want to avoid it here and focus on fully connected graphs with dense Laplacians.

\section{NFFT-based fast summation}
\label{sec::nfft}
For eigenvalue computation as well as various other applications with the graph Laplacian, one needs to perform matrix-vector multiplications with the matrix~$\Ww$ or the matrix $\Aa:=\Dd^{-1/2}\Ww\Dd^{-1/2}$. In general, this requires $\mathcal{O}\big(n^2\big)$ arithmetic operations.
When the matrix~$\Ww$ has entries~\eqref{eq::W_gaussian}, this arithmetic complexity can be reduced to $\mathcal{O}(n)$ using the NFFT-based fast summation \cite{post02,postni04}. In general, this method may be applied when the entries of the matrix $\Ww$ can be written in the form $W_{ji}=K(\vb_j-\vb_i)$, where $K\colon\R^d\rightarrow\mathbb{C}$ is a rotational invariant and smooth kernel function.
For applying the NFFT-based fast summation for \eqref{eq::W_general}, it would be more convenient to consider the matrix $\Ww$ to have entries equal to $K(\mathbf{0})$ on the diagonal and we refer to this matrix as $\tilde{\Ww}$. Note that it can be written as
$\tilde{\Ww}=\Ww + K(\mathbf{0})\, \mathbf{I}$
and thus 
$\Ww=\tilde{\Ww}-K(\mathbf{0})\, \mathbf{I}$.
In order to efficiently compute the row sums of $\Ww$, which appear on the diagonal of $\Dd$, we use
\begin{equation*}
\Ww\mathbf{1}=\tilde{\Ww}\mathbf{1}-K(\mathbf{0})\, \mathbf{I}\mathbf{1}=\tilde{\Ww}\mathbf{1}-K(\mathbf{0})\, \mathbf{1}.
\end{equation*}
We now illustrate how to efficiently compute the matrix-vector product with the matrix $\tilde{\Ww}$ using the NFFT-based fast summation. 
For instance, for the Gaussian kernel function, we have
\begin{equation}
\label{eq::nfftWtilde_x}
\left(\tilde{\Ww}\mathbf{x}\right)_j=\sum_{i=1}^{n}x_i \, \exp \left(-\frac{\norm{\vb_j-\vb_i}^2}{\sigma^2}\right)
\quad\forall j=1,\ldots,n
\end{equation}
with $\mathbf{x}=[x_1,x_2,\ldots,x_n]^T$
and we rewrite~\eqref{eq::nfftWtilde_x} by
\begin{equation}
\label{eq::nfftWtilde_x:kernel}
\left(\tilde{\Ww}\mathbf{x}\right)_j = f(\vb_j) := \sum_{i=1}^{n}x_i \, K(\vb_j-\vb_i)
\end{equation}
with the kernel function
$
K(\mathbf{y}) := \exp (-\norm{\mathbf{y}}^2/\sigma^2).
$
The key idea of the efficient computation of~\eqref{eq::nfftWtilde_x:kernel} is approximating $K$ by a trigonometric polynomial $K_\mathrm{RF}$ in order to separate the computations involving the vertices $\vb_j$ and $\vb_i$.
Assuming we have such a $d$-variate trigonometric polynomial
\begin{equation}\label{eq::K_RF}
K(\mathbf{y})  \approx
K_\mathrm{RF}(\mathbf{y}) := \sum_{\mathbf l\in I_N} \hat{b}_{\mathbf l} \,\mathrm{e}^{2\pi\mathrm{i}\mathbf{l} \mathbf{y}}, \quad
I_N := \{-N/2,-N/2+1,\ldots,N/2-1\}^d,
\end{equation}
with bandwidth $N\in 2\mathbb{N}$ and Fourier coefficients~$\hat{b}_{\mathbf l}$,
we replace $K$ by $K_\mathrm{RF}$ in~\eqref{eq::nfftWtilde_x:kernel} and we obtain
\begin{align*}
\left(\tilde{\Ww}\mathbf{x}\right)_j
 = f(\vb_j)
\approx f_\mathrm{RF}(\mathbf{v}_j)
:=& \sum_{i=1}^{n}x_i \, K_\mathrm{RF}(\mathbf{v}_j-\mathbf{v}_i)
 = \sum_{i=1}^{n} x_i \sum_{\mathbf{l}\in I_N} \hat{b}_{\mathbf l} \,\mathrm{e}^{2\pi\mathrm{i}\mathbf{l} (\mathbf{v}_j-\mathbf{v}_i)} \\
 =& \sum_{\mathbf{l}\in I_N} \hat{b}_{\mathbf l} \left(\sum_{i=1}^{n} x_i \,\mathrm{e}^{-2\pi\mathrm{i}\mathbf{l} \mathbf{v}_i}\right) \,\mathrm{e}^{2\pi\mathrm{i}\mathbf{l} \mathbf{v}_j}, \quad\forall j=1,\ldots,n.
\end{align*}
Using the NFFT~\cite{KeKuPo09}, one computes the inner and outer sums for all $j=1,\ldots,n$ totally in $\mathcal{O}(m^d\,n + N^d\log N)$ arithmetic operations, where $m\in\mathbb{N}$ is an internal window cut-off parameter which influences the accuracy of the NFFT.
Please note that since $K_\mathrm{RF}$ and $f_\mathrm{RF}$ are 1-periodic functions but neither $K$ nor $f$ are, ones needs to shift and scale the nodes~$\vb_j$ such that they are contained in a subset of the cube $[-1/4,1/4]^d$ ensuring $\vb_j-\vb_i\in [-1/2,1/2]^d$.
Depending on the Fourier coefficients~$\hat{b}_{\mathbf l}$, ${\mathbf l}\in I_N$, of the trigonometric polynomial $K_\mathrm{RF}$, where $\hat{b}_{\mathbf l}$ still have to be determined, we may need to scale the nodes~$\vb_j$ to a slightly smaller cube.

We emphasize that we are not restricted to the Gaussian weight function $w(v_j,v_i)=\exp (-\norm{\vb_j-\vb_i}^2/\sigma^2 )$ or a rotational invariant weight function. In fact, any kernel function~$K$ that can be well approximated by a trigonometric polynomial~$K_\mathrm{RF}$ may be used.

Next, we describe an approach to obtain suitable Fourier coefficients~$\hat{b}_{\mathbf l}$ of~$K_\mathrm{RF}$ based on sampling values of~$K$. Especially, we want to obtain a good approximation of~$K$ using a small number of Fourier coefficients~$\hat{b}_{\mathbf l}$.
Therefore, we regularize $K$ to obtain a 1-periodic smooth kernel function~$K_\mathrm{R}$, which is $p-1$ times continuously differentiable (in the periodic setting), such that its Fourier coefficients decay in a fast way. Then, we approximate the Fourier coefficients of~$K_\mathrm{R}$ using the trapezoidal rule and this yields the Fourier coefficients~$\hat{b}_{\mathbf l}$ of $K_\mathrm{RF}$.

For a rotational invariant kernel function~$K(\mathbf{y})$, which is sufficiently smooth except at the ``boundaries'' of the cube $[-1/2,1/2]^d$, e.g.\ $K(\mathbf{y}) = \exp (-\norm{\mathbf{y}}^2/\sigma^2)$,
we only need to regularize near $\norm{\mathbf{y}}=1/2$.
We use the ansatz
\begin{equation*}
K_\mathrm{R}(\mathbf{y}) :=
\begin{cases}
K(\mathbf{y})&\textnormal{ if } \norm{\mathbf{y}}\leq \frac 12-\varepsilon_\mathrm{B}  \\ 
T_\mathrm{B}(\norm{\mathbf{y}}) & \textnormal{ if } \frac 12-\varepsilon_\mathrm{B} < \norm{\mathbf{y}} \leq \frac 12, \\
T_\mathrm{B}\big(\frac 12\big) & \textnormal{ otherwise},
\end{cases}
\qquad
\qquad
\begin{tikzpicture}[baseline={( %
                               current bounding box.center)}]
\begin{axis}[
width=8cm,
height=8cm,
axis x line=none,
axis y line=none,
xmin=-3.5,xmax=2,
ymin=-0.5,ymax=5,
set layers,
mark layer=axis tick labels
]

\addplot[color=black,line width=0.8pt] coordinates {(-1.6,0) (1.55,0)};
\addplot[color=black,line width=0.8pt] coordinates {(-1.6,0) (-1.6,3)};

\addplot[line width=0.8pt,mark=-] coordinates {(-1.6,0.75)} node[left] {\small$-\frac 12 + \varepsilon_\mathrm{B}$};
\addplot[line width=0.8pt,mark=-] coordinates {(-1.6,2.25)} node[left] {\small$\frac 12 - \varepsilon_\mathrm{B}$};
\addplot[line width=0.8pt,mark=-] coordinates {(-1.6,0.25)} node[left] {$-\frac 12$};
\addplot[line width=0.8pt,mark=-] coordinates {(-1.6,2.75)} node[left] {$\frac 12$};

\addplot[color=black,mark=square*,/tikz/mark size=41.5,line width=1pt,mark options={fill=black!25}] coordinates {(0,1.5)};
\addplot[color=black,mark=*,/tikz/mark size=41.4,line width=1pt,mark options={fill=black!15}] coordinates {(0,1.5)};
\addplot[color=black,mark=*,/tikz/mark size=25,line width=1pt,mark options={fill=white}] coordinates {(0,1.5)};

\addplot[color=black,line width=0.8pt,dashed] coordinates {(-1.6,0.25) (1.55,0.25)}; %
\addplot[color=black,line width=0.8pt,dashed] coordinates {(-1.6,0.75) (1.55,0.75)}; %
\addplot[color=black,line width=0.8pt,dashed] coordinates {(-1.6,2.25) (1.55,2.25)}; %
\addplot[color=black,line width=0.8pt,dashed] coordinates {(-1.6,2.75) (1.55,2.75)}; %

\addplot[color=black,line width=0.8pt,dashed] coordinates {(-1.25,0) (-1.25,3)}; %
\addplot[color=black,line width=0.8pt,dashed] coordinates {(-0.75,0) (-0.75,3)}; %
\addplot[color=black,line width=0.8pt,dashed] coordinates {(0.75,0) (0.75,3)}; %
\addplot[color=black,line width=0.8pt,dashed] coordinates {(1.25,0) (1.25,3)}; %

\addplot[color=black,mark=text,/pgf/text mark={\footnotesize $K(\mathbf y)$}] coordinates{(0,1.5)};
\addplot[color=black,mark=text,/pgf/text mark={\footnotesize $T_{\rm B}(\|\mathbf y\|)$}] coordinates{(0,2.45)};

\end{axis}
\end{tikzpicture}
\end{equation*}
where $T_\mathrm{B}$ is a suitably chosen univariate polynomial, e.g. computed by a two-point Taylor interpolation.
 The parameter $0 < \varepsilon_\mathrm{B} \ll 1/2$ determines the size of the regularization region, cf.\ \cite[Sec.~2]{postni04}.
For the treatment of a rotational invariant kernel function which has a singularity at the origin, we also refer to \cite[Sec.~2]{postni04}.
Now we approximate $K_\mathrm{R}$ by the $d$-variate trigonometric polynomial $K_\mathrm{RF}$
from~\eqref{eq::K_RF},
where we compute the
Fourier coefficients
\begin{equation}\label{eq::b_l::K_RF}
\hat{b}_{\mathbf{l}} := \frac{1}{N}\sum_{\mathbf{j}\in I_N} K_\mathrm{R}\left(\frac{\mathbf j}{N}\right) \,\mathrm{e}^{-2\pi\mathrm{i}\mathbf{j} \mathbf{l} / N}
\quad\forall\mathbf{l}\in I_N.
\end{equation}
Assuming one evaluation of $K_\mathrm{R}$ takes $\mathcal{O}(1)$ arithmetic operations,
the computations in~\eqref{eq::b_l::K_RF} require $\mathcal{O}(N^d\log N)$ arithmetic operations in total using the fast Fourier transform.

If all vertices $v_j$ and their corresponding data vectors $\mathbf{v}_j\in\R^d$, $j=1,\ldots,n$, fulfill the property $\norm{\mathbf{v}_j} \leq 1/4-\varepsilon_\mathrm{B}/2$,
we have $\norm{\mathbf{v}_j-\mathbf{v}_i} \leq 1/2-\varepsilon_\mathrm{B}$ and we obtain an approximation of~\eqref{eq::nfftWtilde_x:kernel} by
\begin{equation*}
\left(\tilde{\Ww}\mathbf{x}\right)_j = 
f(\mathbf{v}_j) = f_\mathrm{R}(\mathbf{v}_j) := \sum_{i=1}^{n}x_i K_\mathrm{R}(\mathbf{v}_j-\mathbf{v}_i)
\approx
f_\mathrm{RF}(\mathbf{v}_j) := \sum_{i=1}^{n}x_i K_\mathrm{RF}(\mathbf{v}_j-\mathbf{v}_i).
\end{equation*}
Otherwise, we compute a correction factor $\rho:=(1/4-\varepsilon_\mathrm{B}/2)/\max_{j=1,\ldots,n} \norm{\mathbf{v}_j}$, using transformed vertices $\tilde{\mathbf{v}}_j:=\mathbf{v}_j \rho$, and adjust parameters of the kernel function appropriately. For instance, in case of the Gaussian kernel function, we replace the scaling parameter $\sigma$ by $\tilde{\sigma}:=\sigma \rho$ for the regularized kernel function $K_\mathrm{R}$.

The error of the approximation $f(\mathbf{v}_j):=(\tilde{\Ww}\mathbf{x})_j \approx f_\mathrm{RF}(\mathbf{v}_j)$
depends on the kernel function as well as on the choice of the regularization smoothness $p$, the size of the regularization region~$\varepsilon_\mathrm{B}$, the bandwidth $N$, and the window cut-off parameter $m$. For a fixed accuracy, we fix these parameters $p$, $\varepsilon_\mathrm{B}$, $N$, and~$m$. Hence, for small to medium dimensions~$d$, we obtain a fast approximate algorithm for the matrix-vector multiplication $\tilde{\Ww}\mathbf{x}$ of complexity $\mathcal{O}(n)$, cf.\ Algorithm~\ref{alg::nfft_fastsum}.
This algorithm is implemented as \texttt{applications/fastsum} and \texttt{matlab/fastsum} in C and \textsc{Matlab} within the NFFT3 software library\footnote{\url{https://www.tu-chemnitz.de/~potts/nfft/}}, see also~\cite{KeKuPo09},
and we use the default Kaiser-Bessel window function.
In Figure~\ref{tab:fastsum_param}, we list the relevant control parameters of Algorithm~\ref{alg::nfft_fastsum} and regularization approach~\eqref{eq::b_l::K_RF}.

\begin{algorithm}[htb!]
\caption{Fast approximate matrix-vector multiplication $\tilde{\Ww}\mathbf{x}$ using NFFT-based fast summation,
$(\tilde{\Ww}\mathbf{x})_j=\sum_{i=1}^{n}x_i \,K(\vb_j-\vb_i)$, \, e.g.
$(\tilde{\Ww}\mathbf{x})_j=\sum_{i=1}^{n}x_i \,\exp (-\norm{\mathbf{v}_j-\mathbf{v}_i}^2/\sigma^2),
\quad\forall j=1,\ldots,n$.}
\label{alg::nfft_fastsum}
\vspace{0.5em}
  \begin{tabular}{p{1.1cm}p{3.5cm}p{9.7cm}}
    Input:      %
                & $\left(\hat{b}_{\mathbf{l}}\right)_{\mathbf{l}\in I_N}$ & Fourier coefficients of trigonometric polynomial $K_\mathrm{RF}$ which approximates $K(\yb)$
                  for $\mathbf{y}\in\mathbb{R}^d$, $\norm{\mathbf y}\leq 1/2-\varepsilon_\mathrm{B}$,\newline e.g. obtained by~\eqref{eq::b_l::K_RF}, \\[0.4em]
                & $\left\lbrace \mathbf v_j\right\rbrace_{j=1}^{n}$ & vertex set, $\mathbf v_j\in\mathbb{R}^d$, $\norm{\mathbf{v}_j} \leq 1/4-\varepsilon_\mathrm{B}/2$, \\[0.4em]
                & $\mathbf{x}=[x_1,x_2,\ldots,x_n]^T$  & vector $\in\mathbb{R}^n$. \\
  \end{tabular}
\vspace{1em}
  \begin{enumerate}
  \item Apply $d$-dimensional adjoint NFFT on $\mathbf{x}$ and obtain \newline $\hat{x}_{\mathbf l} \,\approx\, \sum_{i=1}^{n} x_i \,\mathrm{e}^{-2\pi\mathrm{i}\mathbf{l} \mathbf{v}_i} \quad\forall \mathbf{l}\in I_N$.
  \item Multiply result by Fourier coefficients $\left(\hat{b}_{\mathbf l}\right)_{\mathbf{l}\in I_N}$ and obtain $\hat{f}_{\mathbf l} := \hat{b}_{\mathbf l} \, \hat{x}_{\mathbf l} \quad\forall \mathbf{l}\in I_N$.
  \item Apply $d$-dimensional NFFT on $\left(\hat{f}_{\mathbf l}\right)_{\mathbf{l}\in I_N}$ and obtain output \newline $\tilde{f}_\mathrm{RF}(\mathbf{v}_j) \,\approx\, \sum_{\mathbf{l}\in I_N} \hat{f}_{\mathbf{l}} \,\mathrm{e}^{2\pi\mathrm{i}\mathbf{l} \mathbf{v}_j} \quad \forall j=1,\ldots,n$.
  \end{enumerate}
\vspace{1em}
  \begin{tabular}{p{1.4cm}p{3.2cm}p{9.7cm}}
    Output: & $\Big\lbrack \tilde{f}_\mathrm{RF}(\mathbf{v}_j)\Big\rbrack_{j=1,\ldots,n}$ & $\tilde{f}_\mathrm{RF}(\mathbf{v}_j) \approx (\tilde{\Ww}\mathbf{x})_j \quad \forall j=1,\ldots,n$. \\
    \cmidrule{1-3}
 \end{tabular}
  \begin{tabular}{p{2.1cm}p{10.0cm}}
    Complexity: & $\mathcal{O}\big(n\big)$ for fixed accuracy. \\
  \end{tabular}
\end{algorithm}

\begin{figure}
\centering
\begin{tabular}{ll}
\toprule Parameter & Description \\
\midrule
\midrule $N$ & $\in 2\mathbb{N}$ bandwidth (in each dimension) of trigonometric polynomial, \\
& such that $K_\mathrm{RF}\approx K$ \\
\midrule $m$ & $\in \mathbb{N}$ window cut-off parameter of NFFT ($m=8$ gives approximately \\ &  IEEE double precision for default Kaiser-Bessel window) \\
\midrule $p$ & $\in \mathbb{N}$ regularization smoothness for $K_R$ \\ & (default choice $p=m$) \\
\midrule $\varepsilon_\mathrm{B}$ & size of the regularization region, $0\leq \varepsilon_\mathrm{B} \ll 1/2$ \\ & (default choice $\varepsilon_\mathrm{B}=p/N$) \\
\bottomrule
\end{tabular}
\caption{Control parameters for NFFT-based fast summation.\label{tab:fastsum_param}}
\end{figure}

Note that every part of Algorithm~\ref{alg::nfft_fastsum} is deterministic and linear in the input vector $\xx$, i.e., the algorithm constitutes a linear operator that can be written as $\tilde{\Ww} + \Ee$ with an error matrix $\Ee$.
For theoretical error estimates on $\|\Ee \xx\|_\infty = \max_{j} |f(\mathbf{v}_j)-f_\mathrm{RF}(\mathbf{v}_j)|$, we refer to %
\cite{post02,postni04,KuPoSt06}. The basic idea is to start with the estimate
\begin{equation}
\label{eq::KERR_estimate}
|f(\mathbf{v}_j)-f_\mathrm{RF}(\mathbf{v}_j)| \leq \norm{\mathbf{x}}_1 \norm{K_\mathrm{ERR}}_\infty,
\quad
\norm{K_\mathrm{ERR}}_\infty:=\max_{\mathbf{y}\in\mathbb{R}^d, \norm{\mathbf y}\leq 1/2-\varepsilon_\mathrm{B}} |K(\mathbf{y})-K_\mathrm{RF}(\mathbf{y})|,
\end{equation}
caused by the approximation of the kernel $K$ by a trigonometric polynomial $K_\mathrm{RF}$, and to additionally take the errors caused by the NFFT into account.
In practice, one may guess $\norm{K_\mathrm{ERR}}_\infty$ based on sampling values of $K$ and $K_\mathrm{RF}$.
For theoretical error estimates of the NFFT for various window functions, we refer to \cite{KeKuPo09,Ne14}.
In practice, choosing the window cut-off parameter of the NFFT $m=8$ yields approximately IEEE double precision for the default Kaiser-Bessel window, see e.g.\ \cite[Sec.~5.2]{KeKuPo09}.

We again emphasize that Algorithm~\ref{alg::nfft_fastsum} is not restricted to the Gaussian kernel function. Any kernel function~$K$ that can be well approximated by a trigonometric polynomial may be used and the corresponding Fourier coefficients~$\hat{b}_{\mathbf l}$, $\mathbf{l}\in I_N$, are an input parameter of Algorithm~\ref{alg::nfft_fastsum}.

Moreover, for the Gaussian kernel function, one could also use the analytic Fourier coefficients~$\hat{b}_{\mathbf l}$ from~\cite{KuPoSt06} for small values of the scaling parameter~$\sigma$ instead of computing~$\hat{b}_{\mathbf l}$ by interpolation in~\eqref{eq::b_l::K_RF}.
In this case, explicit error bounds for $\norm{K_\mathrm{ERR}}_\infty$ are available.

\subsection{Error propagation for normalized matrices}
\label{sec::nfft::propagation}
As seen in Section~\ref{sec::graph_laplacian}, many applications involving the Graph Laplacian require matrix vector products with a matrix $\Aa$ that itself does not follow the form of \eqref{eq::W_general}, but results from normalization of such a matrix~$\Ww$. This normalization can be written as $\Aa = \Dd^{-1/2} \Ww \Dd^{-1/2}$, where $\Dd = \mathrm{diag}(\Ww \mathbf{1})$. Since our approach includes replacing all matrix-vector products $\Ww \xx$ by the approximations \linebreak $(\tilde{\Ww} + \Ee) \xx - K(\mathbf{0})\, \xx$, this also includes the computation of the degree matrix $\Dd$. The error occurring from this approximation will then propagate to the evaluation error of $\Aa \xx$.

\begin{algorithm}[htb!]
\caption{Fast approximate matrix-vector multiplication $\Aa\mathbf{x}$ using NFFT-based fast summation, with $\Aa = \Dd^{-1/2} \Ww \Dd^{-1/2}$, $\Ww$ as in \eqref{eq::W_general} and $\Dd = \mathrm{diag}(\Ww\mathbf{1})$}
\label{alg::nfft_normalized_matvec}
\vspace{0.5em}
\begin{tabular}{p{1.1cm}p{3.4cm}p{9.7cm}}
	Input:
	& $\sigma$ or $c$, $\{\vb_j\}_{j=1}^n$ & Scaling parameter and vertex set, $\vb_j \in \mathbb{R}^d$, specifying $\Ww$, \\
	& $\mathbf{x}=[x_1,\ldots,x_n]^T$  & vector $\in\mathbb{R}^n$.
\end{tabular}
\vspace{1em}
\begin{enumerate}
	\item Choose correction factor $\rho$ such that $\|\rho \vb_j\|_2 \leq 1/4-\varepsilon_\mathrm{B}/2$ for all $j=1,\ldots,n$. \\
	      Set $\vb_j:=\rho \vb_j$ for all $j=1,\ldots,n$.
	\item For Gaussian and Laplacian RBF kernel, adjust scaling parameter $\sigma:=\rho\sigma$. \\
	      For multiquadric and inverse multiquadric kernel, adjust parameter $c:=c/\rho$.
	\item For the computation of matrix-vector products with the matrix
		\[\tilde{\Ww}_\Ee = %
		\tilde{\Ww} + \Ee = K(\mathbf{0})\, \mathbf{I} + \Ww + \Ee \]
		by Algorithm~\ref{alg::nfft_fastsum},
	    determine appropriate control parameters for the NFFT-based fast summation, see Figure~\ref{tab:fastsum_param}, and
	    obtain Fourier coefficients $\hat{b}_{\mathbf{l}}$, $\mathbf{l}\in I_N$, %
	    e.g.\ by~\eqref{eq::b_l::K_RF} or~\cite{KuPoSt06}.
	\item Compute $\Dd_\Ee = \mathrm{diag}(\tilde{\Ww}_\Ee \mathbf{1} - K(\mathbf{0})\, \mathbf{1}) \approx \mathrm{diag}(\Ww \mathbf{1}) = \Dd$ via Algorithm~\ref{alg::nfft_fastsum} \\
	      (scale output of Algorithm~\ref{alg::nfft_fastsum} by $\rho$ for multiquadric kernel and $1/\rho$ for inverse multiquadric kernel).
	\item Compute $\yb = \Dd_\Ee^{-1/2} \left( \tilde{\Ww}_\Ee(\Dd_\Ee^{-1/2} \mathbf{x}) - K(\mathbf{0})\, \Dd_\Ee^{-1/2} \mathbf{x}\right) \approx \Aa \xx $ via Algorithm~\ref{alg::nfft_fastsum} \\
          (scale output of Algorithm~\ref{alg::nfft_fastsum} by $\rho$ for multiquadric kernel and $1/\rho$ for inverse multiquadric kernel).
\end{enumerate}
\vspace{1em}
\begin{tabular}{p{1.4cm}p{3.2cm}p{9.7cm}}
	Output: & $\yb$ & Approximate result of $\Aa \mathbf{x}$. \\
	\cmidrule{1-3}
\end{tabular}
\begin{tabular}{p{2.1cm}p{10.0cm}}
	Complexity: & $\mathcal{O}\big(n\big)$ for fixed accuracy. \\
\end{tabular}
\end{algorithm}

Algorithm \ref{alg::nfft_normalized_matvec} summarizes the usage of Algorithm \ref{alg::nfft_fastsum} for this case. Note that if multiple matrix-vector products are required, e.g. in an iterative scheme, steps 1--4 can be performed once in a setup phase. The following lemma gives an estimation of the error of Algorithm~\ref{alg::nfft_normalized_matvec} depending on the relative error of Algorithm~\ref{alg::nfft_fastsum}.

\begin{lemma} \label{thm::normalization_propagation}
Let $\Ww \in \mathbb{R}^{n\times n}$ be a matrix with non-negative entries and at least one positive entry per row. Given an error matrix $\Ee \in \mathbb{R}^{n \times n}$, we define $\Ww_\Ee = \Ww + \Ee$ and
\begin{align*}
[d_1,\ldots,d_n]^T &:= \Ww \mathbf{1}, & \Dd &:= \mathrm{diag}(d_1, \ldots, d_n), & \Aa &:= \Dd^{-1/2} \Ww \Dd^{-1/2}, \\
[d_{\Ee,1},\ldots,d_{\Ee,n}]^T &:= \Ww_\Ee \mathbf{1}, & \Dd_\Ee &:= \mathrm{diag}(d_{\Ee,1}, \ldots, d_{\Ee,n}) & \Aa_\Ee &:= \Dd_\Ee^{-1/2} \Ww_\Ee \Dd_\Ee^{-1/2}.
\end{align*}
	Let $d_{\min} > 0$ denote the minimum diagonal entry of $\Dd$ and furthermore set
	\[ \eta := \frac{d_{\min}}{\|\Ww\|_\infty} \qquad \text{and} \qquad \varepsilon := \frac{\|\Ee\|_\infty}{\|\Ww\|_\infty}. \]
	Then, for $\varepsilon < \eta$, it holds
	\begin{equation*}
	\|\Aa - \Aa_\Ee\|_\infty \leq \frac{\varepsilon (1+\eta)}{\eta(\eta-\varepsilon)}.
	\end{equation*}
\end{lemma}

\begin{proof}
Due to
\[ |d_i - d_{\Ee,i}| \leq \|\Ww\mathbf{1} - \Ww_\Ee\mathbf{1}\|_\infty = \| \Ee \mathbf{1} \|_\infty \leq \|\Ee\|_\infty \|\mathbf{1}\|_\infty = \|\Ee\|_\infty = \varepsilon \|\Ww\|_\infty \]
and the fact that $x \mapsto x^{-1/2}$ and its first derivative are monotoneously decreasing, we obtain
\begin{align*}
\|\Dd^{-1/2} - \Dd_\Ee^{-1/2}\|_\infty &= \max_{i} \big|d_i^{-1/2} - d_{\Ee,i}^{-1/2}\big| \\
&\leq \max_{i} %
\max_{-\|\Ee\|_\infty \leq \delta \leq \|\Ee\|_\infty} 
\big|d_i^{-1/2} - (d_i + \delta)^{-1/2}\big| \\
&= \max_i \big|d_i^{-1/2} - (d_i - \|\Ee\|_\infty)^{-1/2}\big| \\
&= \big|d_{\min}^{-1/2} - (d_{\min} - \|\Ee\|_\infty)^{-1/2}\big| \\
&= \big|\eta^{-1/2} - (\eta - \varepsilon)^{-1/2}\big| \ \|\Ww\|_\infty^{-1/2} \\
&= \Big((\eta - \varepsilon)^{-1/2} - \eta^{-1/2}\Big) \|\Ww\|_\infty^{-1/2}.
\end{align*}
Analogously we obtain
\[ \|\Dd_\Ee^{-1/2}\|_\infty \leq (\eta-\varepsilon)^{-1/2}\|\Ww\|_\infty^{-1/2}. \]
Together with $\|\Dd^{-1/2}\|_\infty = \eta^{-1/2}\|\Ww\|_\infty^{-1/2}$ and $\|\Ww_\Ee\|_\infty \leq (1+\varepsilon)\|\Ww\|_\infty$, this yields
\begin{align*}
\|\Aa - \Aa_\Ee\|_\infty &= \Big\|(\Dd^{-1/2} - \Dd_\Ee^{-1/2}) \Ww \Dd^{-1/2} \\
&\qquad + \Dd_\Ee^{-1/2} (\Ww - \Ww_\Ee) \Dd^{-1/2} \\
&\qquad\qquad + \Dd_\Ee^{-1/2} \Ww_\Ee (\Dd^{-1/2} - \Dd_\Ee^{-1/2}) \Big\|_\infty \\
&\leq \Big((\eta - \varepsilon)^{-1/2} - \eta^{-1/2}\Big) \|\Ww\|_\infty^{-1/2} \; \|\Ww\|_\infty \; \eta^{-1/2}\|\Ww\|_\infty^{-1/2} \\
&\qquad + (\eta-\varepsilon)^{-1/2} \|\Ww\|_\infty^{-1/2} \; \varepsilon \|\Ww\|_\infty \; \eta^{-1/2}\|\Ww\|_\infty^{-1/2} \\
&\qquad\qquad + (\eta-\varepsilon)^{-1/2} \|\Ww\|_\infty^{-1/2} \; (1+\varepsilon) \|\Ww\|_\infty \; \Big((\eta - \varepsilon)^{-1/2} - \eta^{-1/2}\Big) \|\Ww\|_\infty^{-1/2} \\
&= \Big((\eta - \varepsilon)^{-1/2} - \eta^{-1/2}\Big) \Big(\eta^{-1/2} + (1+\underline{\varepsilon)(\eta - \varepsilon)^{-1/2}}\Big) + \underline{(\eta - \varepsilon)^{-1/2}  \varepsilon} \eta^{-1/2}. \\
\intertext{Now detach the left underlined part from its paranthesed expression and combine it with the right underlined part:}
&= \Big((\eta - \varepsilon)^{-1/2} - \eta^{-1/2}\Big) \Big((\eta - \varepsilon)^{-1/2} + \eta^{-1/2}\Big) \\
&\qquad\qquad + \underline{\varepsilon (\eta - \varepsilon)^{-1/2}} \Big((\eta - \varepsilon)^{-1/2} - \eta^{-1/2} + \eta^{-1/2}\Big) \\
\intertext{Resolve the binomial expression in the first line and simplify the second line:} 
&= (\eta - \varepsilon)^{-1} - \eta^{-1} + \varepsilon (\eta - \varepsilon)^{-1} = \frac{(1+\varepsilon) \eta - (\eta-\varepsilon)}{\eta (\eta - \varepsilon)} = \frac{\varepsilon(1 + \eta)}{\eta (\eta - \varepsilon)}.
\end{align*}
This concludes the proof for the desired inequality.
\end{proof}

The requirement $\varepsilon < \eta$ means that $\|\Ee\|_\infty$ must be smaller than the smallest diagonal entry in $\Dd$. This condition cannot be avoided since otherwise, negative entries in $\Dd_\Ee$ could not be ruled out, leading to imaginary entries in $\Dd_\Ee^{-1/2}$ and thus in $\Aa_\Ee$. On the other hand, if $\varepsilon$ is well below $\eta$, Lemma~\ref{thm::normalization_propagation} yields that the absolute error in $\Aa$ is linear in $\varepsilon$, which is the relative error of Algorithm~\ref{alg::nfft_fastsum}. 

Alternatively, by ignoring the error caused by the NFFT, we obtain error estimations of the form
\begin{equation}
\label{eq::epsilon_KERR_estimate}
\|\Ee \xx \|_\infty \lesssim \|K_\mathrm{ERR}\|_\infty \|\xx\|_1 \leq n \, \|K_\mathrm{ERR}\|_\infty \|\xx\|_\infty \qquad \Rightarrow \qquad \varepsilon = \frac{\|\Ee\|_\infty}{\|\Ww\|_\infty} \lesssim  n \, \frac{\|K_\mathrm{ERR}\|_\infty}{\|\Ww\|_\infty}.
\end{equation}
In other words, the perturbation grows linearly in the size of the dataset. If either $\|\Ww\|_\infty$ or $d_{\min}$ grew less fast, then Lemma~\ref{thm::normalization_propagation} would not be applicable for large $n$ because $\varepsilon$ would eventually supersede $\eta$. However, if we assume that increasing $n$ means adding more similarly-distributed data points to the dataset, the \emph{average} entry in $\Ww$ does not change and thus all row sums of $\Ww$ also grow linearly in $n$, including $d_{\min}$ and the maximum row sum $\|\Ww\|_\infty$. A mathematical quantification of this observation is beyond the scope of this article, but in practice, the values for $\eta$ and $\varepsilon$ can be approximated and monitored to give a-posteriori error bounds. One way to do this is by using \eqref{eq::epsilon_KERR_estimate} and approximating $\|K_\mathrm{ERR}\|_\infty$ via \eqref{eq::KERR_estimate}, where the maximum can be discretized in a large number of randomly drawn sample points. 
The accuracy of this approximation can be validated by explicitly computing the exact absolute row sum $\|\Ee\|_\infty$ via
\begin{equation}
\label{eq::epsilon_norm_aposteriori}
\|\Ee\|_\infty = \left\| \sum_{i=1}^n \left| \Ee \mathbf{e}_i \right| \right\|_\infty = \left\|\sum_{i=1}^n \left| \tilde{\Ww}_\Ee \mathbf{e}_i - \Ww \mathbf{e}_i - K(\mathbf{0}) \mathbf{e}_i \right| \right\|_\infty, 
\end{equation}
where $|\cdot|$ is applied elementwise, $\mathbf{e}_i$ denotes the $i$-th unit vector, and matrix-vector products with $\tilde{\Ww}_\Ee = \Ww + K(\mathbf{0}) \mathbf{I} + \Ee$ are evaluated using Algorithm~\ref{alg::nfft_fastsum}. The effort of computing \eqref{eq::epsilon_norm_aposteriori} is $\mathcal{O}(n^2)$.
Equivalently, the true value for $\|\Aa - \Aa_\Ee\|_\infty$ can be computed via
\begin{equation*}
\|\Aa - \Aa_\Ee\|_\infty = \left\| \sum_{i=1}^n \left| \Aa \mathbf{e}_i - \Aa_\Ee \mathbf{e}_i \right| \right\|_\infty.
\end{equation*}

\color{black}

\section{Krylov subspace methods and NFFT}
\label{sec::lanczos}

The main contribution of this paper is the usage of NFFT-based fast summation for accelerating Krylov subspace methods, which are the state-of-the-art schemes for the solution of linear equation systems, eigenvalue problems, and more \cite{book::saad}. In the case of large dense matrices, the computational bottleneck is the setup of and multiplication with the system matrix itself. We will here exemplarily illustrate this for the Lanczos algorithm \cite{lanczos}, which is the standard method for computation of a few dominating, i.e. largest, eigenvalues of a symmetric matrix $\Aa$ \cite{book::parlett, book::golubvanloan}. It is based on looking for an $\Aa$-invariant subspace in the Krylov space 
\begin{equation*}
	\mathcal{K}_k(\Aa,\rb)=\mathrm{span}\left\lbrace \rb,\Aa \rb,\Aa^2\rb,\Aa^3\rb,\ldots,\Aa^{k-1}\rb\right\rbrace.
\end{equation*}
This is achieved by iteratively constructing an orthonormal basis $\qb_1,\ldots,\qb_k$ of this space in such a way that the matrix $\Qq_k = [\qb_1,\ldots,\qb_k] \in \R^{n\times k}$ yields a tridiagonalization of $\Aa$, i.e.
\begin{equation*}
	\Qq_k^T \Aa \Qq_k = \Tt_k =
	\begin{bmatrix}
		\alpha_1 & \beta_2 & &\\
		\beta_2 & \alpha_2 & \ddots &\\
		& \ddots & \ddots & \beta_k \\
		& & \beta_k & \alpha_k
	\end{bmatrix}.
\end{equation*}
Such a matrix $\Qq_k$ as well as the entries of $\Tt_k$ can be computed by the iteration
\[
	\qb_1 = \frac{\mathbf r}{\|\mathbf{r}\|}, \qquad\quad \qb_{k+1} = \frac{1}{\beta_{k+1}} \left( \Aa \qb_k - \alpha_k \qb_k - \beta_k \qb_{k-1} \right) \quad \forall k = 1,2,\ldots
\]
where $\alpha_k = \qb_k^T \Aa \qb_k$ and $\beta_{k+1} = \|\Aa\qb_k - \alpha_k\qb_k - \beta_k\qb_{k-1}\|$. The remarkable fact that this iteration produces orthonormal vectors is a consequence of the symmetry of $\Aa$. 
We now summarize the first $k$ steps of the Lanczos process in the relation
\begin{equation}
\label{krylov1}
\Aa\Qq_k=\Qq_k\Tt_k+\beta_{k+1}\qb_{k+1}\mathbf{e}_{k}^{T},
\end{equation}
where $\mathbf{e}_j$ denotes the $j$-th standard basis vector of the appropriate dimension.
The eigenvalues and eigenvectors of the small matrix $\Tt_k$ are called the Ritz values and vectors, respectively, and can be computed efficiently. From $\Tt_k \mathbf{w} = \lambda \mathbf{w}$ we then obtain
\[
	\Aa \Qq_k \mathbf{w} = \Qq_k \Tt_k \mathbf{w} + \beta_{k+1} \qb_{k+1} \mathbf{e}_k^T \mathbf{w} = \lambda \Qq_k \mathbf{w} + \beta_{k+1} w_k \qb_{k+1},
\]
where $w_k$ is the $k$-th component of the Ritz vector $\mathbf{w}$. We finally see via
\[
	\|\Aa \Qq_k \mathbf{w} - \lambda\Qq_k \mathbf{w}\| = |\beta_{k+1} \mathbf{w}_k| \leq |\beta_{k+1}|
\]
that a small value $|\beta_{k+1}|$ indicates that $(\lambda, \Qq_k \mathbf{w})$ is a good approximation to an eigenpair of~$\Aa$ and that the Krylov space is close to containing an $\Aa$-invariant subspace.
There are many more practical issues that make the implementation of the Lanczos process more efficient and robust. We do not discuss these points in detail but refer to \cite{book::parlett,lehoucq1998arpack} for the details.

Additionally, we want to point out that the above procedure can also be used for the solution of linear systems of equations. Standard methods based on the Lanczos method are the conjugate gradients method \cite{cg} and the minimal residual method \cite{minres}, which are tailored for the solution of linear systems of the form $\mathbf{A}\mathbf{x}=\mathbf{b}$. Note that such applications involving the graph Laplacian are commonly found in kernel based methods \cite{christopher2016pattern}. In the nonsymmetric case that comes up e.g. when considering $\Ll_w$, we can employ the Arnoldi method \cite{book::saad}, which relies on a similar iteration where $\Tt_k$ is replaced by an upper Hessenberg matrix.

One main contribution of this paper is the fact that by evaluating matrix-vector products via the NFFT-based Algorithms~\ref{alg::nfft_fastsum} or \ref{alg::nfft_normalized_matvec}, Krylov subspace methods are still applicable for dense matrices that are too large to store, let alone apply, as long as they stem from the kernel structure of \eqref{eq::W_general} or normalization of such a matrix.
In our experiments, this method will be denoted as NFFT-based Lanczos method.

A detailed discussion of the effect of inexact matrix-vector products on Krylov-based approximations can be found in \cite{simoncini2003theory}.

\section{Alternative eigenvalue algorithm: The Nystr\"om method}
\label{sec::nystrom}

\subsection{The traditional Nystr\"om extension}
\label{sec::nystrom::traditional}

The Nystr\"om extension is currently used as a method of choice to compute eigenvalue approximations of kernel-based matrices that are too large to allow for direct eigenvalue computation. See e.g. \cite{GaMeBeFlPe14} and \cite{MeKoBe13} for its applications in different settings. Originally introduced to the matrix computations context in \cite{WS01nystrom}, further improvements have been suggested in \cite{FowBCM04} and \cite{drineas2005nystrom} and its usage for classification problems has been proposed in \cite{BerF12}. It is based on dividing the data points into a sample set $X$ of $L$ nodes and its complement $Y$. After permutation, the adjacency matrix $\Ww$ can be split into blocks
\[
\Ww=
 \begin{bmatrix} \Ww_{XX} & \Ww_{XY} \\ \Ww_{XY}^{T} & \Ww_{YY} \end{bmatrix},
\]
where the blocks $\Ww_{XX}\in \mathbb{R}^{L\times L}$ and $\Ww_{YY} \in \mathbb{R}^{(n-L)\times(n-L)}$ are the adjacency matrices of the canonical subgraphs with node sets $X$ and $Y$, respectively, and the block $\Ww_{XY} \in \mathbb{R}^{L \times (n-L)}$ contains the similarities between all combinations of nodes from $X$ and $Y$.

The basic idea of the Nystr\"om method is to compute only $\Ww_{XX}$ and $\Ww_{XY}$ explicitly, but not the remaining block $\Ww_{YY}$.
If $L\ll n$, the approach significantly decreases the required number of data point comparisons. Assuming that $\Ww_{XX}$ is regular, the method approximates~$\Ww$ by
\begin{equation} \label{eq:nystromapprox}
\Ww \approx \Ww_\Ee = \begin{bmatrix} \Ww_{XX} \\ \Ww_{XY}^T \end{bmatrix} \Ww_{XX}^{-1} \begin{bmatrix} \Ww_{XX} & \Ww_{XY} \end{bmatrix} = \begin{bmatrix} \Ww_{XX} & \Ww_{XY} \\ \Ww_{XY}^T & \Ww_{XY}^{T} \Ww_{XX}^{-1} \Ww_{XY} \end{bmatrix},
\end{equation}
which constitutes a rank-$L$ approximation due to the size and regularity of $\Ww_{XX}$. This formula is used once in approximating the degree matrix $\Dd$ by $\Dd_\Ee = \diag(\Ww_\Ee \mathbf{1})$ and once in approximating the eigenvalues of $\Aa$ via the rank-$L$ eigenvalue decomposition

\[
\Aa_\Ee := \Dd_\Ee^{-1/2} \, \Ww_\Ee \, \Dd_\Ee^{-1/2} = \Vv_L \mathbf{\Lambda}_L \Vv_L^*.
\]
This can be computed without having to set up the full matrix, e.g. by the technique described in \cite{FowBCM04} made up mainly of two singular value decompositions of $(L\times L)$-sized matrices, which is technically only applicable if $\Ww$ is positive definite. Alternatively, we have achieved better results by computing the QR factorization $\mathbf{\hat Q}\mathbf{\hat R} := \Dd_\Ee^{-1/2} [\Ww_{XX} \; \Ww_{XY}]^T$ and the eigenvalue decomposition $\Uu_L \mathbf{\Lambda}_L\Uu_L^T := \mathbf{\hat R}\Ww_{XX}^{-1}\mathbf{\hat R}^T$, leading to the eigenvector matrix $\Vv_L = \mathbf{\hat Q} \Uu_L$.
The arithmetic complexity of this algorithm can be easily confirmed to be $\mathcal{O}\big(n\,L^2)$.

The eigenvalue accuracy depends strongly on the quality of the approximation
\[ \Ww_{YY} \approx \Ww_{XY}^T \Ww_{XX}^{-1} \Ww_{XY}. \]
Since the sample set $X$ is a randomly chosen subset of the indices from $1,\ldots,n$,
its size $L$ is the decisive method parameter and its choice is a nontrivial task. On the one hand, $L$ needs to be small for the method to be efficient. On the other hand, a too small choice of $L$ may cause extreme errors, especially because the approximation error in $\Dd_\Ee$ propagates to the eigenvalue computation.
In spite of the positivity of the diagonal of $\Dd$, negative entries in $\Dd_\Ee$ cannot be ruled out and are observed in practice. Hence imaginary entries may occur in $\Dd_\Ee^{-1/2}$ and thus $\Aa_\Ee$, making the results extremely unreliable. This behaviour follows the same structure as Lemma~\ref{thm::normalization_propagation}, however, we do not have a meaningful bound on $\|\Ww_{YY} - \Ww_{XY}^T \Ww_{XX}^{-1} \Ww_{XY}\|_\infty$ that would guarantee favorable error behaviour.

\subsection{A NFFT-based accelerated Nystr\"om-Gaussian method}
\label{sec::nystrom::gaussian_nfft}

Another important contribution of this paper is the development of an improved Nystr\"om method, which utilizes the NFFT-based fast summation from Section~\ref{sec::nfft}.
It is based on a slightly different algorithm that has been recently introduced as a Nystr\"om method, cf.\ \cite{Martinsson18} and the references therein.
Their basic idea is rewriting the traditional Nystr\"om approximation as
\[
 \Aa \approx (\Aa \Qq) (\Qq^T \Aa \Qq)^{-1} (\Aa \Qq)^T%
\]
where $\Qq \in \mathbb{R}^{n \times L}$ is a matrix with orthogonal columns. 
If $\Qq$ holds the first $L$ columns of a permutation matrix, one obtains the traditional Nystr\"om method from Section~\ref{sec::nystrom::traditional}.
Inspired by similar randomized linear algebra algorithm such as randomized singular value decomposition,
this choice of $\Qq$ is replaced in~\cite{Martinsson18} by $\Qq = \mathrm{orth}(\Aa \mathbf{G})$, where $\mathbf{G} \in \mathbb{R}^{n \times L}$ is a Gaussian matrix with normally distributed random entries and $\mathrm{orth}$ denotes column-wise orthonormalization. Unfortunately, this setup requires $2L$ matrix-vector products with the full matrix $\Aa$.

We now propose accelerating these matrix-vector products by computing $\Aa \Qq$ column-wise via the NFFT-based fast summation Algorithm~\ref{alg::nfft_fastsum} in order to avoid full matrix setup or slow direct matrix-vector products. In addition, we propose replacing the inverse $(\Qq^T \Aa \Qq)^{-1}$ by a low-rank approximation based only on the $M\in\mathbb{N}$ largest eigenvalues of $\Qq^T \Aa \Qq$. This way, a rank-$M$ approximation of $\Aa$ is produced, where $M$ may be the actual number of required eigenvalues or larger. The resulting method ``Nystr\"om-Gaussian-NFFT'' is presented in Algorithm~\ref{alg::nystrom_hybrid}.
Its arithmetic complexity is $\mathcal{O}\big(n\,L^2\big)$. On the first glance, this arithmetic complexity seems to be identical to the one of the traditional Nystr\"om method from Section~\ref{sec::nystrom::traditional}. However, as we observe in the numerical tests in Section~\ref{sec::res::ev}, we may choose the parameter~$L$ distinctly smaller for  Algorithm~\ref{alg::nystrom_hybrid}, i.e., $L\sim k$, where $k$ is the number of eigenvalues and eigenvectors. Then, the resulting arithmetic complexity is $\mathcal{O}\big(n\,k^2\big)$.

\begin{algorithm}[htb!]
\caption{NFFT-based accelerated Nystr\"om-Gaussian method (``Nystr\"om-Gaussian-NFFT'') for eigenvalue approximation $\Vv_k \mathbf{\Lambda}_k \Vv_k^* \approx \Aa := \Dd^{-1/2} \Ww \Dd^{-1/2}$.}\label{alg::nystrom_hybrid}
\vspace{0.5em}
  \begin{tabular}{p{1.4cm}p{2.6cm}p{10cm}}
    Input:      %
		        & $\sigma$, $\left\lbrace \mathbf v_j\right\rbrace_{j=1}^{n}$ & scaling parameter and vertex set, $\mathbf v_j\in\mathbb{R}^d$, specifying $\Ww$, \\
                & $k$& number of desired eigenvalues $\in\mathbb{N}$, \\
                & $L$& number of random Gaussian columns $\geq M\geq k$, \\
                & $M$& rank of inversion $\geq k$.
  \end{tabular}
\vspace{1em}
  \begin{enumerate}
  \item Setup the NFFT-based fast summation parameters for computing matrix-vector products with $\Ww$ using Algorithm~\ref{alg::nfft_fastsum}, cf.\ Section~\ref{sec::nfft}.
  \item Compute the degree matrix $\Dd = \diag(\Ww \mathbf{1})$ using Algorithm~\ref{alg::nfft_fastsum}. %
  \item Setup a random Gaussian matrix $\mathbf{G} \in \mathbb{R}^{n \times L}$, compute $\mathbf{Y} = \Aa \mathbf{G}$ column-wise via Algorithm~\ref{alg::nfft_fastsum}, and
	$\Qq = \mathrm{orth}(\mathbf{Y}) \in \R^{n\times L}$ by QR-factorization.
  \item Compute $\mathbf{B}_1 = \Aa \Qq \in \R^{n\times L}$ column-wise via Algorithm~\ref{alg::nfft_fastsum} and $\mathbf{B}_2 = \Qq^T \mathbf{B}_1 \in \R^{L\times L}$.
  \item Compute the diagonal matrix $\mathbf{\Sigma}_M$ of the $M$ largest positive eigenvalues of $\mathbf{B_2}$ and the matrix $\Uu_M \in \R^{L\times M}$ holding the corresponding orthonormal eigenvectors as columns. %
  \item Compute the QR-factorization $\mathbf{\hat Q} \mathbf{\hat R} = \mathbf{B}_1 \Uu_M$, \; $\mathbf{\hat Q} \in \R^{n\times M}$, $\mathbf{\hat R} \in \R^{M\times M}$.
  \item Compute the eigenvalue decomposition $\hat{\Uu}_M \mathbf{\Lambda}_M \hat{\Uu}_M^T = \mathbf{\hat R} \mathbf{\Sigma}_M^{-1} \mathbf{\hat R}^T$ and set $\Vv_M = \hat{\Qq} \hat{\Uu}_M$.
  \item Put the $k$ largest eigenvalues from $\mathbf{\Lambda}_M$ into the diagonal matrix~$\mathbf{\Lambda}_k$ and corresponding eigenvectors from $\Vv_M$ into $\Vv_k$.
  \end{enumerate}
\vspace{1em}
  \begin{tabular}{p{1.4cm}p{2.6cm}p{10.35cm}}
    Output: & $\mathbf{\Lambda}_k \in \mathbb{R}^{k \times k}$ & diagonal matrix of approximated largest eigenvalues of $\Aa$, \\
	    & $\Vv_k \in \mathbb{R}^{n \times k}$ & corresponding orthonormal eigenvector matrix. \\
    \cmidrule{1-3}
 \end{tabular}
   \begin{tabular}{p{2.1cm}p{10.0cm}}
     Complexity: & $\mathcal{O}\big(n L^2\big)$ \\
   \end{tabular}
\end{algorithm}

\section{Numerical results}
\label{sec::res}

All our experiments are performed using \textsc{Matlab} implementations based on the NFFT3 library and \textsc{Matlab}'s \texttt{eigs} function. A short example code can be found on the homepage of the authors.\footnote{\url{https://www.tu-chemnitz.de/mathematik/wire/people/files_alfke/NFFT-Lanczos-Example-v1.tar.gz}}

\subsection{Accuracy and runtime of eigenvalue computations}
\label{sec::res::ev}

\begin{sloppypar}
We use the function \texttt{generateSpiralDataWithLabels.m}\footnote{\url{https://sites.google.com/site/kittipat/matlabtechniques}} to generate varying sets of three-dimensional data. The data points are in the form of a spiral and we can specify the number of classes as well as the number of points per class. We generate data sets with 5 classes and equal numbers of points per class, which have a total number of data points $n\in\{2\,000, 5\,000, 10\,000, 20\,000, 50\,000, 100\,000\}$. For the generation, we use the default parameters $h=10$ and $r=2$ in \texttt{generateSpiralDataWithLabels.m}.
For each~$n$, we generate 5 random spiral data sets.
In Figure~\ref{fig::illustration_datasets::spiral}, we visualize an example data set with $n=2\,000$ total points.
For the adjacency matrix~$\Ww$, we set the scaling parameter $\sigma=3.5$.
Using the NFFT-based Lanczos method from Section~\ref{sec::lanczos}, we compute the 10 largest eigenvalues and the corresponding eigenvectors of the matrix $\Aa:=\Dd^{-1/2}\Ww\Dd^{-1/2}$ for each data set.
We consider three different parameter setups for the NFFT in Algorithm~\ref{alg::nfft_fastsum}, achieving different accuracies. We set the bandwidth $N=16$ and the window cut-off parameter $m=2$ in setup \#1, $N=32$ and $m=4$ in setup \#2, as well as $N=64$ and $m=7$ in setup \#3. For all three setups, we use $\varepsilon_\mathrm{B}=0$.
For comparison, we also apply the Nystr\"om method from Section~\ref{sec::nystrom::traditional}, where we perform 10 repetitions for each data set, since the method uses random sub-sampling in order to obtain a rank-$L$ approximation of the adjacency matrix $\Ww$. We consider two different Nystr\"om setups with rank $L\in\{n/10, n/4\}$.
Moreover, we use the hybrid Nystr\"om-Gaussian-NFFT method from Algorithm~\ref{alg::nystrom_hybrid} in Section~\ref{sec::nystrom::gaussian_nfft} with $L\in\{20,50\}$ Gaussian columns, parameter $M=10$ as well as fast summation parameters corresponding to setup~\#2, where we perform 10 repetitions for each data set.
Additionally, we compute the eigenvalues and eigenvectors by a direct method, which applies the Lanczos method using full matrix-vector products with the adjacency matrix $\Ww$.
For the Nystr\"om method from Section~\ref{sec::nystrom::traditional} and the direct computation method, we only run tests for a total number of data points $n\in\{2\,000, 5\,000, 10\,000, 20\,000\}$ due to long runtimes.
\end{sloppypar}

\begin{figure}[htb!]
\centering{
\subfloat[Spiral example with $n=2\,000$ points. %
]{\label{fig::illustration_datasets::spiral}
\begin{tikzpicture}[baseline]
  \begin{axis}[view={130}{25},enlargelimits=false,font=\footnotesize,width=0.44\textwidth,height=0.45\textwidth,
    scatter/classes={
      1={mark=*,mark size=0.25,magenta,style={solid, fill=magenta}},
      2={mark=*,mark size=0.25,blue,style={solid, fill=blue}},
      3={mark=*,mark size=0.25,draw=green,style={solid, fill=green}},
      4={mark=*,mark size=0.25,draw=cyan,style={solid, fill=cyan}},
      5={mark=*,mark size=0.25,draw=red,style={solid, fill=red}}
    },
  ]
    \addplot3[scatter,only marks,scatter src=explicit symbolic] file {Plots/spiral_2000.txt};
  \end{axis} 
\end{tikzpicture}
}
\hfill
\subfloat[Crescent-fullmoon example with $n=100\,000$ points.]{\label{fig::illustration_datasets::crescentfullmoon}
\hspace{2em}
\begin{tikzpicture}[baseline]
  \begin{axis}[enlargelimits={abs=0.5},font=\footnotesize,width=0.5\textwidth,
    axis equal=true,
  ]
  \addplot graphics[xmin=-8.3,xmax=8.3,ymin=-8.3,ymax=5.3] {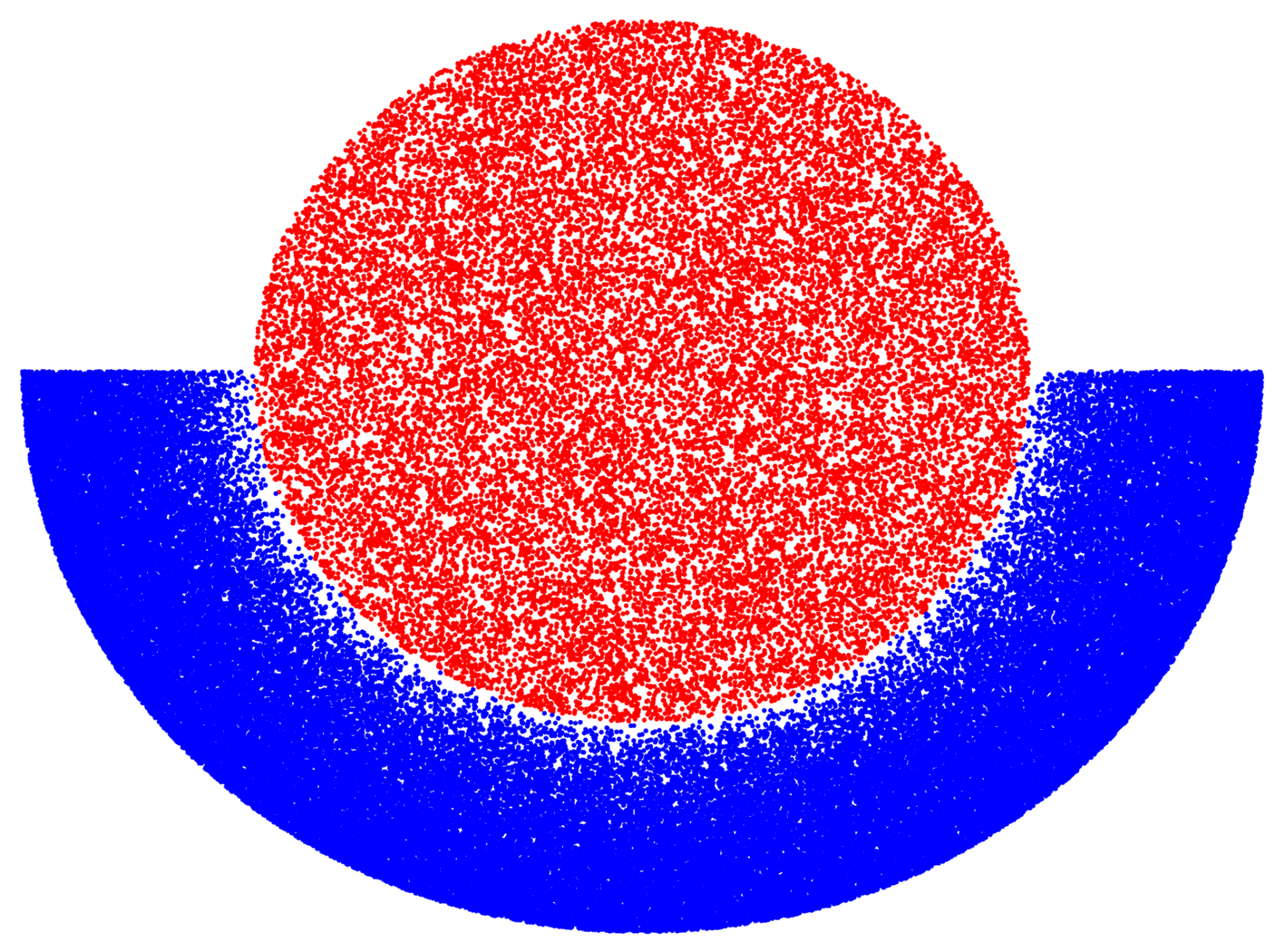};
  \end{axis}
\end{tikzpicture}
\hspace{2em}
}
}
\caption{Illustration of spiral and crescent-fullmoon data sets. 
}\label{fig:illustration_datasets}
\end{figure}
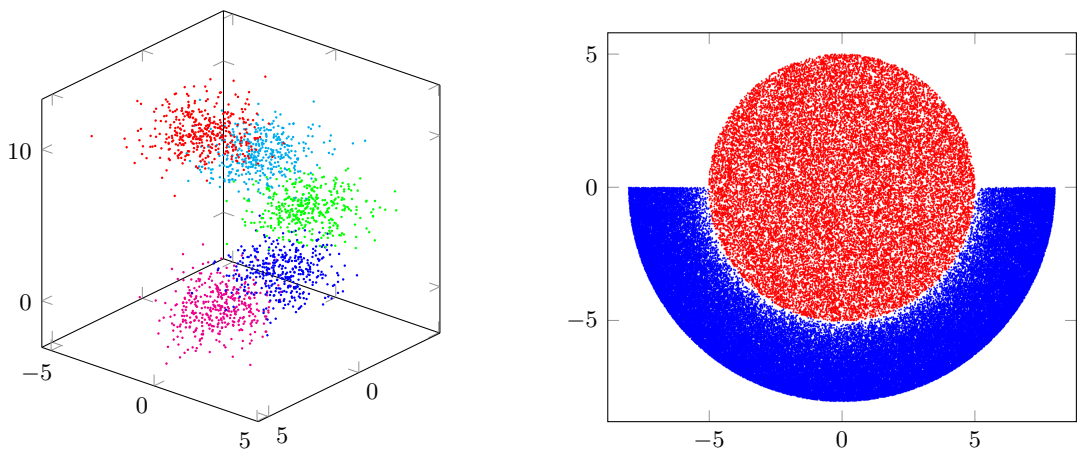

\begin{figure}[htb!]
\centering{
\subfloat[Comparison of eigenvalue accuracies.]{\label{fig::compare_ev_residuals::eigenvalues}
\begin{tikzpicture}[baseline]
  \begin{axis}[font=\footnotesize,enlargelimits=true,xmin=2000,xmax=50000,height=0.5\textwidth, width=0.47\textwidth, grid=major, xlabel={$n$}, ylabel={Min / avg. / max. of \begin{NoHyper}\eqref{equ::max_ev_error}\end{NoHyper}},
    xtick={2000,5000,10000,20000,50000},xticklabel={\pgfkeys{/pgf/fpu=true}\pgfmathparse{exp(\tick)}\pgfmathprintnumber[1000 sep={\,},fixed relative, precision=3]{\pgfmathresult}\pgfkeys{/pgf/fpu=false}},
    ytick={1e0,1e-1,1e-2,1e-3,1e-4,1e-5,1e-9,1e-10,1e-14,1e-15},
    legend style={legend cell align=left, align=left, at={(0.4,1.2)}, anchor=south},%
    xmode=log,ymode=log,
    xmajorgrids=false,
  ]
    \addplot[blue,mark=triangle*,mark size=1,thick,mark options={solid},
                error bars/.cd, y dir=plus, y explicit, error bar style={solid}, error mark options={rotate=90,mark size=2,thick}] coordinates {
      (2000,9.499e-01) -= (0,9.251e-01) += (0,2.178e+01)  (5000,6.390e-01) -= (0,6.281e-01) += (0,2.144e+01)  (10000,3.562e-02) -= (0,3.145e-02) += (0,3.089e-01)  (20000,2.587e-01) -= (0,2.567e-01) += (0,9.304e+00)
    };
    \addplot[cyan,dashed,mark=triangle*,mark size=1,thick,mark options={solid},
                    error bars/.cd, y dir=both, y explicit, error bar style={solid}, error mark options={rotate=90,mark size=2,thick}] coordinates {
      (2000,1.466e+00) -= (0,1.456e+00) += (0,5.372e+01)  (5000,6.535e-02) -= (0,6.271e-02) += (0,1.806e+00)  (10000,8.805e-02) -= (0,8.667e-02) += (0,3.462e+00)  (20000,1.749e-02) -= (0,1.683e-02) += (0,2.006e-01)
    };
    \addplot[darkgreen,mark=triangle*,mark size=1,thick,mark options={solid,rotate=180},
                        error bars/.cd, y dir=both, y explicit, error bar style={solid}, error mark options={rotate=90,mark size=2,thick}] coordinates {
      (2000,4.395e-03) -= (0,3.575e-03) += (0,1.121e-02)  (5000,2.631e-03) -= (0,2.244e-03) += (0,4.089e-03)  (10000,2.612e-03) -= (0,2.123e-03) += (0,5.665e-03)  (20000,2.048e-03) -= (0,1.689e-03) += (0,5.421e-03)  (50000,1.659e-03) -= (0,1.246e-03) += (0,2.428e-03)
    };
    \addplot[darkgreen,mark=triangle*,dashed,mark size=1,thick,mark options={solid,rotate=180},
                        error bars/.cd, y dir=both, y explicit, error bar style={solid}, error mark options={rotate=90,mark size=2,thick}] coordinates {
      (2000,5.259e-05) -= (0,4.604e-05) += (0,1.127e-04)  (5000,3.677e-05) -= (0,2.358e-05) += (0,4.265e-05)  (10000,2.802e-05) -= (0,1.690e-05) += (0,5.305e-05)  (20000,1.598e-05) -= (0,1.078e-05) += (0,1.968e-05)  (50000,7.345e-06) -= (0,4.853e-06) += (0,1.659e-05)
    };
    \addplot[red,mark=*,mark size=1,thick,mark options={solid},
                        error bars/.cd, y dir=both, y explicit, error bar style={solid}, error mark options={rotate=90,mark size=2,thick}] coordinates {
      (2000,9.451e-05) -= (0,8.367e-05) += (0,1.989e-04)  (5000,1.049e-04) -= (0,7.427e-05) += (0,1.893e-04)  (10000,1.595e-04) -= (0,6.495e-05) += (0,8.507e-05)  (20000,2.609e-04) -= (0,1.623e-04) += (0,1.911e-04)  (50000,3.275e-04) -= (0,1.790e-04) += (0,2.862e-04)
    };
    \addplot[red,dashed,mark=*,mark size=1,thick,mark options={solid},
                            error bars/.cd, y dir=both, y explicit, error bar style={solid}, error mark options={rotate=90,mark size=2,thick}] coordinates {
      (2000,2.848e-10) -= (0,8.795e-11) += (0,1.671e-10)  (5000,1.834e-10) -= (0,4.858e-11) += (0,2.883e-11)  (10000,1.718e-10) -= (0,7.715e-11) += (0,1.783e-10)  (20000,8.173e-11) -= (0,3.193e-11) += (0,3.552e-11)  (50000,6.653e-11) -= (0,3.210e-11) += (0,4.687e-11)
    };
    \addplot[red,dashdotted,mark=*,mark size=1,thick,mark options={solid},
                            error bars/.cd, y dir=both, y explicit, error bar style={solid}, error mark options={rotate=90,mark size=2,thick}] coordinates {
      (2000,4.494e-15) -= (0,2.162e-15) += (0,4.499e-15)  (5000,3.331e-15) -= (0,6.661e-16) += (0,1.443e-15)  (10000,3.086e-15) -= (0,7.550e-16) += (0,1.465e-15)  (20000,3.375e-15) -= (0,1.155e-15) += (0,9.548e-16)  (50000,2.642e-15) -= (0,1.310e-15) += (0,1.132e-15)
    };
    \addplot[amber,mark=pentagon*,mark size=1.5,thick,mark options={solid},
                        error bars/.cd, y dir=both, y explicit, error bar style={solid}, error mark options={mark size=2,thick,rotate=90}] coordinates {
  (2000,5.466e-04) -= (0,1.071e-04) += (0,1.342e-04)  (5000,4.650e-04) -= (0,2.721e-04) += (0,2.169e-04)  (10000,4.211e-04) -= (0,2.311e-04) += (0,9.142e-05)  (20000,1.900e-04) -= (0,2.632e-05) += (0,4.215e-05)  (50000,1.648e-04) -= (0,2.284e-05) += (0,5.352e-05)
    };
    \addplot[amber,dashed,mark=pentagon*,mark size=1.5,thick,mark options={solid},
                        error bars/.cd, y dir=both, y explicit, error bar style={solid}, error mark options={mark size=2,thick,rotate=90}] coordinates {
  (2000,2.376e-15) -= (0,1.266e-15) += (0,9.548e-16)  (5000,1.465e-08) -= (0,3.465e-09) += (0,5.349e-09)  (10000,1.465e-08) -= (0,6.437e-09) += (0,5.296e-09)  (20000,1.088e-08) -= (0,5.740e-09) += (0,5.549e-09)  (50000,9.464e-09) -= (0,3.987e-09) += (0,5.428e-09)
    };
    \addplot[amber,dashdotted,mark=pentagon*,mark size=1.5,thick,mark options={solid},
                        error bars/.cd, y dir=both, y explicit, error bar style={solid}, error mark options={mark size=2,thick,rotate=90}] coordinates {
  (2000,2.376e-15) -= (0,1.488e-15) += (0,2.176e-15)  (5000,3.253e-15) -= (0,1.588e-15) += (0,2.187e-15)  (10000,1.761e-14) -= (0,1.317e-14) += (0,1.925e-14)  (20000,3.009e-14) -= (0,1.099e-14) += (0,1.821e-14)  (50000,3.633e-14) -= (0,1.890e-14) += (0,3.406e-14)
    };
  \end{axis}
\end{tikzpicture}
}
\subfloat[Comparison of eigenvector accuracies.]{\label{fig::compare_ev_residuals::residuals}
\begin{tikzpicture}[baseline]
  \begin{axis}[font=\footnotesize,enlargelimits=true,xmin=2000,xmax=100000,height=0.5\textwidth, width=0.47\textwidth, grid=major, xlabel={$n$},
    ylabel={Min / avg. / max. of \begin{NoHyper}\eqref{equ::max_res_norm}\end{NoHyper}},
    xtick={2000,5000,20000,100000},xticklabel={\pgfkeys{/pgf/fpu=true}\pgfmathparse{exp(\tick)}\pgfmathprintnumber[1000 sep={\,},fixed relative, precision=3]{\pgfmathresult}\pgfkeys{/pgf/fpu=false}},
    ytick={1e0,1e-1,1e-2,1e-3,1e-8,1e-13,1e-15},
    legend style={legend cell align=left, align=left, at={(0.4,1.2)}, anchor=south},%
    xmode=log,ymode=log,
    xmajorgrids=false,
  ]
    \addplot[blue,mark=triangle*,mark size=1,thick,mark options={solid},
                error bars/.cd, y dir=both, y explicit, error bar style={solid}, error mark options={rotate=90,mark size=2,thick}] coordinates {
      (2000,1.244e+00) -= (0,1.201e+00) += (0,2.246e+01)  (5000,9.272e-01) -= (0,8.912e-01) += (0,2.208e+01)  (10000,1.094e-01) -= (0,1.033e-01) += (0,1.224e+00)  (20000,3.850e-01) -= (0,3.815e-01) += (0,1.017e+01)
    };
    \addplot[cyan,dashed,mark=triangle*,mark size=1,thick,mark options={solid},
                    error bars/.cd, y dir=both, y explicit, error bar style={solid}, error mark options={rotate=90,mark size=2,thick}] coordinates {
      (2000,1.747e+00) -= (0,1.720e+00) += (0,5.438e+01)  (5000,1.511e-01) -= (0,1.467e-01) += (0,2.714e+00)  (10000,1.539e-01) -= (0,1.510e-01) += (0,4.393e+00)  (20000,6.415e-02) -= (0,6.273e-02) += (0,6.206e-01)
    };
    \addplot[darkgreen,mark=triangle*,mark size=1,thick,mark options={solid,rotate=180},
                        error bars/.cd, y dir=both, y explicit, error bar style={solid}, error mark options={rotate=90,mark size=2,thick}] coordinates {
      (2000,9.971e-03) -= (0,5.270e-03) += (0,7.455e-03)  (5000,7.796e-03) -= (0,4.606e-03) += (0,5.584e-03)  (10000,7.680e-03) -= (0,4.318e-03) += (0,5.634e-03)  (20000,6.740e-03) -= (0,4.397e-03) += (0,8.120e-03)  (50000,6.144e-03) -= (0,3.059e-03) += (0,4.851e-03)  (100000,6.534e-03) -= (0,3.572e-03) += (0,8.317e-03)
    };
    \addplot[darkgreen,mark=triangle*,dashed,mark size=1,thick,mark options={solid,rotate=180},
                        error bars/.cd, y dir=both, y explicit, error bar style={solid}, error mark options={rotate=90,mark size=2,thick}] coordinates {
      (2000,1.551e-03) -= (0,4.925e-04) += (0,7.529e-04)  (5000,9.254e-04) -= (0,3.131e-04) += (0,4.176e-04)  (10000,6.941e-04) -= (0,2.195e-04) += (0,5.720e-04)  (20000,4.833e-04) -= (0,1.894e-04) += (0,3.019e-04)  (50000,2.983e-04) -= (0,1.363e-04) += (0,2.866e-04)  (100000,2.124e-04) -= (0,9.225e-05) += (0,2.145e-04)
    };
    \addplot[red,mark=*,mark size=1,thick,mark options={solid},
                        error bars/.cd, y dir=both, y explicit, error bar style={solid}, error mark options={rotate=90,mark size=2,thick}] coordinates {
      (2000,1.407e-04) -= (0,7.442e-05) += (0,2.063e-04)  (5000,1.403e-04) -= (0,7.346e-05) += (0,2.037e-04)  (10000,1.977e-04) -= (0,7.151e-05) += (0,9.473e-05)  (20000,3.090e-04) -= (0,1.775e-04) += (0,2.072e-04)  (50000,3.823e-04) -= (0,1.959e-04) += (0,3.136e-04)  (100000,4.517e-04) -= (0,2.554e-04) += (0,3.627e-04)
    };
    \addplot[red,dashed,mark=*,mark size=1,thick,mark options={solid},
                            error bars/.cd, y dir=both, y explicit, error bar style={solid}, error mark options={rotate=90,mark size=2,thick}] coordinates {
      (2000,9.821e-09) -= (0,7.785e-10) += (0,5.290e-10)  (5000,9.723e-09) -= (0,5.507e-10) += (0,3.984e-10)  (10000,9.433e-09) -= (0,2.075e-10) += (0,2.562e-10)  (20000,9.143e-09) -= (0,3.851e-10) += (0,4.027e-10)  (50000,9.044e-09) -= (0,4.643e-10) += (0,3.754e-10)  (100000,8.947e-09) -= (0,5.178e-10) += (0,4.367e-10)
    };
    \addplot[red,dashdotted,mark=*,mark size=1,thick,mark options={solid},
                            error bars/.cd, y dir=both, y explicit, error bar style={solid}, error mark options={rotate=90,mark size=2,thick}] coordinates {
      (2000,6.206e-14) -= (0,5.193e-14) += (0,1.490e-13)  (5000,1.325e-14) -= (0,2.904e-15) += (0,8.110e-15)  (10000,1.089e-14) -= (0,2.771e-16) += (0,2.277e-16)  (20000,1.104e-14) -= (0,2.076e-16) += (0,2.607e-16)  (50000,1.269e-14) -= (0,2.288e-16) += (0,1.727e-16)  (100000,1.478e-14) -= (0,2.981e-16) += (0,2.839e-16)
    };
    \addplot[black,mark=square*, mark size=1,thick,
                   error bars/.cd, y dir=both, y explicit, error bar style={solid,thick}, error mark options={rotate=90,mark size=2,thick}] coordinates {
      (2000,2.684e-15) -= (0,7.260e-16) += (0,1.075e-15)  (5000,3.806e-15) -= (0,9.936e-16) += (0,2.074e-15)  (10000,4.314e-15) -= (0,4.228e-16) += (0,2.376e-16)  (20000,5.559e-15) -= (0,3.773e-16) += (0,7.955e-16)  (50000,8.656e-15) -= (0,1.605e-16) += (0,1.296e-16)
    };
    \addplot[amber,mark=pentagon*,mark size=1.5,thick,mark options={solid},
                        error bars/.cd, y dir=both, y explicit, error bar style={solid}, error mark options={mark size=2,thick,rotate=90}] coordinates {
  (2000,5.630e-04) -= (0,1.024e-04) += (0,1.321e-04)  (5000,4.859e-04) -= (0,2.520e-04) += (0,2.087e-04)  (10000,4.380e-04) -= (0,2.258e-04) += (0,9.031e-05)  (20000,2.123e-04) -= (0,2.397e-05) += (0,3.487e-05)  (50000,1.807e-04) -= (0,2.623e-05) += (0,4.294e-05)  (100000,1.865e-04) -= (0,2.112e-05) += (0,1.832e-05)
    };
    \addplot[amber,dashed,mark=pentagon*,mark size=1.5,thick,mark options={solid},
                        error bars/.cd, y dir=both, y explicit, error bar style={solid}, error mark options={mark size=2,thick,rotate=90}] coordinates {
  (2000,8.128e-12) -= (0,7.813e-12) += (0,3.048e-11)  (5000,1.602e-08) -= (0,2.561e-09) += (0,4.985e-09)  (10000,1.563e-08) -= (0,6.769e-09) += (0,4.853e-09)  (20000,1.285e-08) -= (0,5.503e-09) += (0,4.879e-09)  (50000,1.207e-08) -= (0,7.057e-09) += (0,9.456e-09)  (100000,1.139e-08) -= (0,5.523e-09) += (0,6.787e-09)
    };
    \addplot[amber,dashdotted,mark=pentagon*,mark size=1.5,thick,mark options={solid},
                        error bars/.cd, y dir=both, y explicit, error bar style={solid}, error mark options={mark size=2,thick,rotate=90}] coordinates {
  (2000,1.934e-15) -= (0,5.382e-17) += (0,9.030e-17)  (5000,2.974e-15) -= (0,3.811e-16) += (0,6.375e-16)  (10000,2.756e-14) -= (0,1.616e-14) += (0,2.062e-14)  (20000,4.811e-14) -= (0,1.979e-14) += (0,1.964e-14)  (50000,6.328e-14) -= (0,2.637e-14) += (0,6.678e-14)  (100000,4.420e-14) -= (0,1.664e-14) += (0,2.159e-14)
    };
  \end{axis}
\end{tikzpicture}
}
\\[0.5em]
\begin{tikzpicture} 
    \begin{axis}[
    font=\footnotesize,
    hide axis,
    xmin=10,
    xmax=50,
    ymin=0,
    ymax=0.4,
    legend columns=3, 
    legend style={legend cell align=left}
    ]
    \addlegendimage{red,mark=*,mark size=2,thick,mark options={solid}}
    \addlegendentry{NFFT-Lanczos, setup \#1}
    \addlegendimage{red,dashed,mark=*,mark size=2,thick,mark options={solid}}
    \addlegendentry{NFFT-Lanczos, setup \#2}
    \addlegendimage{red,dashdotted,mark=*,mark size=2,thick,mark options={solid}}
    \addlegendentry{NFFT-Lanczos, setup \#3\!}
    \addlegendimage{blue,mark=triangle*,mark size=2,thick,mark options={solid}}
    \addlegendentry{Nystr\"om, $L=n/10$}
    \addlegendimage{cyan,dashed,mark=triangle*,mark size=2,thick,mark options={solid}}
    \addlegendentry{Nystr\"om, $L=n/4$}
    \addlegendimage{black,mark=square*, mark size=2,thick}
    \addlegendentry{Direct method}
    \addlegendimage{amber,mark=pentagon*,mark size=2,thick,mark options={solid}}
    \addlegendentry{FIGTree-Lanczos, $\epsilon=$5e-3}
    \addlegendimage{amber,dashed,mark=pentagon*,mark size=2,thick,mark options={solid}}
    \addlegendentry{FIGTree-Lanczos, $\epsilon=$2e-6\;}
    \addlegendimage{amber,dashdotted,mark=pentagon*,mark size=2,thick,mark options={solid}}
    \addlegendentry{FIGTree-Lanczos, $\epsilon=$1e-10}
    \addlegendimage{darkgreen,mark=triangle*,mark size=2,thick,mark options={solid,rotate=180}}
    \addlegendentry{Nystr\"om-Gaussian-NFFT, $L=20$\!\!\!\!\!\!\!\!\!\!\!\!\!\!\!\!\!}
    \addlegendimage{empty legend}
    \addlegendentry{}
    \addlegendimage{darkgreen,dashed,mark=triangle*,mark size=2,thick,mark options={solid,rotate=180}}
    \addlegendentry{Nystr\"om-Gaussian-NFFT, $L=50$}
    \end{axis}
\end{tikzpicture}
\\[0.5em]
\subfloat[Residuals for $n=20\,000$.]{\label{fig::compare_ev_residuals::residuals20000}
\begin{tikzpicture}[baseline]
  \begin{axis}[font=\footnotesize,enlargelimits=true,xmin=1,xmax=10,height=0.49\textwidth, width=0.48\textwidth, grid=major, xlabel={Eigenvalue}, ylabel={Avg. and max. residual norms},
    xtick={1,2,3,4,5,6,7,8,9,10},
    ytick={1e0,1e-3,1e-8,1e-13,1e-15},
    legend columns=3, legend style={legend cell align=left, align=left, at={(0.5,1.2)}, anchor=south},
    ymode=log,
    xmajorgrids=false,
  ]
    \addplot[blue,mark=triangle*,mark size=1,thick,mark options={solid},
                error bars/.cd, y dir=plus, y explicit, error bar style={solid}, error mark options={rotate=90,mark size=2,thick}] coordinates {
      (1,2.626e-01) -= (0,2.621e-01) += (0,1.029e+01)  (2,2.911e-02) -= (0,2.812e-02) += (0,6.997e-01)  (3,3.599e-02) -= (0,3.442e-02) += (0,6.055e-01)  (4,4.418e-02) -= (0,4.209e-02) += (0,4.399e-01)  (5,5.380e-02) -= (0,5.153e-02) += (0,2.296e-01)  (6,2.804e-02) -= (0,2.531e-02) += (0,1.649e-01)  (7,1.863e-02) -= (0,1.614e-02) += (0,1.603e-01)  (8,2.100e-02) -= (0,1.832e-02) += (0,9.000e-02)  (9,1.432e-02) -= (0,1.155e-02) += (0,7.939e-02)  (10,1.779e-02) -= (0,1.472e-02) += (0,9.989e-02)
    };
    \addplot[cyan,dashed,mark=triangle*,mark size=1,thick,mark options={solid},
                    error bars/.cd, y dir=plus, y explicit, error bar style={solid}, error mark options={rotate=90,mark size=2,thick}] coordinates {
      (1,1.142e-03) -= (0,9.320e-04) += (0,1.537e-02)  (2,3.244e-03) -= (0,2.846e-03) += (0,4.408e-02)  (3,2.922e-02) -= (0,2.866e-02) += (0,6.555e-01)  (4,2.202e-02) -= (0,2.137e-02) += (0,4.384e-01)  (5,2.698e-03) -= (0,1.854e-03) += (0,1.156e-02)  (6,1.197e-02) -= (0,1.112e-02) += (0,1.468e-01)  (7,4.205e-03) -= (0,3.288e-03) += (0,2.598e-02)  (8,1.055e-02) -= (0,9.427e-03) += (0,1.106e-01)  (9,4.403e-03) -= (0,3.376e-03) += (0,2.769e-02)  (10,5.900e-03) -= (0,4.885e-03) += (0,5.587e-02)
    };
    \addplot[darkgreen,mark=triangle*,mark size=1,thick,mark options={solid,rotate=180},
                        error bars/.cd, y dir=plus, y explicit, error bar style={solid}, error mark options={rotate=90,mark size=2,thick}] coordinates {
      (1,7.080e-04) -= (0,5.004e-04) += (0,1.391e-03)  (2,7.889e-04) -= (0,5.352e-04) += (0,1.723e-03)  (3,1.162e-03) -= (0,8.003e-04) += (0,1.691e-03)  (4,1.729e-03) -= (0,1.223e-03) += (0,6.177e-03)  (5,2.891e-03) -= (0,1.962e-03) += (0,3.228e-03)  (6,3.137e-03) -= (0,1.622e-03) += (0,3.030e-03)  (7,4.190e-03) -= (0,2.464e-03) += (0,5.262e-03)  (8,3.551e-03) -= (0,2.241e-03) += (0,2.136e-03)  (9,4.904e-03) -= (0,3.389e-03) += (0,9.956e-03)  (10,5.498e-03) -= (0,4.129e-03) += (0,6.712e-03)
    };
    \addplot[darkgreen,mark=triangle*,dashed,mark size=1,thick,mark options={solid,rotate=180},
                        error bars/.cd, y dir=plus, y explicit, error bar style={solid}, error mark options={rotate=90,mark size=2,thick}] coordinates {
      (1,4.391e-05) -= (0,2.677e-05) += (0,4.432e-05)  (2,5.100e-05) -= (0,2.611e-05) += (0,2.694e-05)  (3,7.270e-05) -= (0,3.574e-05) += (0,6.776e-05)  (4,1.182e-04) -= (0,5.930e-05) += (0,1.156e-04)  (5,1.910e-04) -= (0,9.409e-05) += (0,1.960e-04)  (6,2.855e-04) -= (0,9.572e-05) += (0,1.787e-04)  (7,2.994e-04) -= (0,1.839e-04) += (0,1.555e-04)  (8,3.343e-04) -= (0,1.420e-04) += (0,2.248e-04)  (9,3.719e-04) -= (0,1.625e-04) += (0,4.133e-04)  (10,4.262e-04) -= (0,2.003e-04) += (0,3.508e-04)
    };
    \addplot[red,mark=*,mark size=1,thick,mark options={solid},
                        error bars/.cd, y dir=plus, y explicit, error bar style={solid}, error mark options={rotate=90,mark size=2,thick}] coordinates {
      (1,6.143e-05) -= (0,1.248e-05) += (0,1.860e-05)  (2,1.421e-04) -= (0,6.452e-05) += (0,9.852e-05)  (3,1.708e-04) -= (0,7.777e-05) += (0,9.307e-05)  (4,2.159e-04) -= (0,1.295e-04) += (0,1.805e-04)  (5,3.090e-04) -= (0,1.775e-04) += (0,2.072e-04)  (6,1.042e-04) -= (0,5.125e-05) += (0,6.108e-05)  (7,8.455e-05) -= (0,3.676e-05) += (0,5.087e-05)  (8,1.071e-04) -= (0,5.168e-05) += (0,5.110e-05)  (9,1.149e-04) -= (0,5.977e-05) += (0,8.667e-05)  (10,1.567e-04) -= (0,7.537e-05) += (0,8.535e-05)
    };
    \addplot[red,dashed,mark=*,mark size=1,thick,mark options={solid},
                            error bars/.cd, y dir=plus, y explicit, error bar style={solid}, error mark options={rotate=90,mark size=2,thick}] coordinates {
      (1,9.143e-09) -= (0,3.851e-10) += (0,4.027e-10)  (2,5.792e-09) -= (0,3.275e-10) += (0,3.962e-10)  (3,4.527e-09) -= (0,6.987e-11) += (0,4.493e-11)  (4,3.602e-09) -= (0,1.569e-10) += (0,1.166e-10)  (5,2.688e-09) -= (0,1.602e-10) += (0,1.145e-10)  (6,2.720e-09) -= (0,1.834e-10) += (0,1.109e-10)  (7,2.412e-09) -= (0,9.143e-11) += (0,4.053e-11)  (8,2.448e-09) -= (0,8.777e-11) += (0,1.198e-10)  (9,2.469e-09) -= (0,1.293e-10) += (0,1.822e-10)  (10,1.852e-09) -= (0,5.510e-11) += (0,2.625e-11)
    };
    \addplot[red,dashdotted,mark=*,mark size=1,thick,mark options={solid},
                            error bars/.cd, y dir=plus, y explicit, error bar style={solid}, error mark options={rotate=90,mark size=2,thick}] coordinates {
      (1,1.104e-14) -= (0,2.076e-16) += (0,2.607e-16)  (2,8.840e-15) -= (0,4.532e-16) += (0,6.558e-16)  (3,5.359e-15) -= (0,8.950e-17) += (0,1.174e-16)  (4,4.023e-15) -= (0,1.414e-16) += (0,2.162e-16)  (5,2.489e-15) -= (0,1.842e-16) += (0,9.978e-17)  (6,2.178e-15) -= (0,6.356e-17) += (0,1.005e-16)  (7,2.006e-15) -= (0,9.739e-17) += (0,2.128e-16)  (8,2.022e-15) -= (0,8.227e-17) += (0,1.043e-16)  (9,2.029e-15) -= (0,9.337e-17) += (0,1.570e-16)  (10,1.779e-15) -= (0,1.914e-16) += (0,2.195e-16)
    };
    \addplot[amber,mark=pentagon*,mark size=1.5,thick,mark options={solid},
                        error bars/.cd, y dir=plus, y explicit, error bar style={solid}, error mark options={rotate=90,mark size=2,thick}] coordinates {
  (1,5.516e-05) -= (0,6.864e-06) += (0,7.896e-06)  (2,2.123e-04) -= (0,2.397e-05) += (0,3.487e-05)  (3,8.401e-05) -= (0,1.825e-05) += (0,9.821e-06)  (4,9.035e-05) -= (0,1.162e-05) += (0,1.773e-05)  (5,6.516e-05) -= (0,1.023e-05) += (0,1.583e-05)  (6,4.063e-05) -= (0,3.618e-06) += (0,5.627e-06)  (7,3.480e-05) -= (0,7.373e-06) += (0,8.300e-06)  (8,4.158e-05) -= (0,8.024e-06) += (0,4.486e-06)  (9,4.182e-05) -= (0,9.258e-06) += (0,8.238e-06)  (10,4.109e-05) -= (0,6.515e-06) += (0,8.178e-06)
    };
    \addplot[amber,dashed,mark=pentagon*,mark size=1.5,thick,mark options={solid},
                        error bars/.cd, y dir=plus, y explicit, error bar style={solid}, error mark options={rotate=90,mark size=2,thick}] coordinates {
  (1,7.987e-09) -= (0,4.385e-09) += (0,3.671e-09)  (2,1.261e-08) -= (0,5.261e-09) += (0,5.120e-09)  (3,7.745e-09) -= (0,2.939e-09) += (0,4.854e-09)  (4,6.642e-09) -= (0,2.141e-09) += (0,3.746e-09)  (5,3.103e-09) -= (0,1.446e-09) += (0,1.014e-09)  (6,3.433e-09) -= (0,8.132e-10) += (0,8.489e-10)  (7,2.621e-09) -= (0,1.233e-09) += (0,1.075e-09)  (8,3.468e-09) -= (0,7.676e-10) += (0,1.427e-09)  (9,2.681e-09) -= (0,2.083e-10) += (0,2.577e-10)  (10,2.379e-09) -= (0,1.171e-09) += (0,9.407e-10)
    };
    \addplot[amber,dashdotted,mark=pentagon*,mark size=1.5,thick,mark options={solid},
                        error bars/.cd, y dir=plus, y explicit, error bar style={solid}, error mark options={rotate=90,mark size=2,thick}] coordinates {
  (1,3.096e-14) -= (0,1.619e-14) += (0,1.710e-14)  (2,4.489e-14) -= (0,1.658e-14) += (0,2.285e-14)  (3,3.670e-14) -= (0,2.529e-14) += (0,2.795e-14)  (4,2.810e-14) -= (0,4.396e-15) += (0,6.218e-15)  (5,2.355e-14) -= (0,1.302e-14) += (0,2.130e-14)  (6,2.511e-14) -= (0,1.210e-14) += (0,1.121e-14)  (7,2.221e-14) -= (0,8.693e-15) += (0,1.960e-14)  (8,2.783e-14) -= (0,1.635e-14) += (0,1.261e-14)  (9,2.302e-14) -= (0,1.112e-14) += (0,1.191e-14)  (10,1.981e-14) -= (0,1.228e-14) += (0,1.862e-14)
    };
  \end{axis}
\end{tikzpicture}
}
\subfloat[Comparison of runtimes.]{\label{fig::compare_ev_residuals::runtimes}
\begin{tikzpicture}[baseline]
  \begin{axis}[font=\footnotesize,enlargelimits=true,xmin=2000,xmax=100000,height=0.49\textwidth, width=0.48\textwidth, grid=major, xlabel={$n$}, ylabel={Avg. and max. runtimes (s)},
    xtick={2000,5000,20000,100000},xticklabel={\pgfkeys{/pgf/fpu=true}\pgfmathparse{exp(\tick)}\pgfmathprintnumber[1000 sep={\,},fixed relative, precision=3]{\pgfmathresult}\pgfkeys{/pgf/fpu=false}},
    legend style={legend cell align=left, align=left, at={(0.4,1.2)}, anchor=south},
    xmode=log,ymode=log,
    xmajorgrids=false,
  ]
    \addplot[black,mark=square*, mark size=1,thick,
    error bars/.cd, y dir=plus, y explicit, error bar style={solid,thick}, error mark options={rotate=90,mark size=2,thick}] coordinates {
      (2000,9.219e+00) -= (0,1.317e-01) += (0,1.115e-01)  (5000,5.444e+01) -= (0,3.306e-01) += (0,2.322e-01)  (10000,2.153e+02) -= (0,2.014e+00) += (0,5.086e+00)  (20000,8.546e+02) -= (0,8.293e+01) += (0,6.135e+01)  (50000,5.401e+03) -= (0,6.097e+02) += (0,2.828e+02)
    };
    \addplot[blue,mark=triangle*,mark size=1,thick,mark options={solid},
                error bars/.cd, y dir=plus, y explicit, error bar style={solid}, error mark options={rotate=90,mark size=2,thick}] coordinates {
      (2000,1.451e-01) -= (0,1.905e-02) += (0,1.106e-01)  (5000,1.530e+00) -= (0,1.433e-01) += (0,1.601e+00)  (10000,9.546e+00) -= (0,4.132e-01) += (0,1.306e+01)  (20000,7.043e+01) -= (0,5.209e+00) += (0,9.997e+01)
    };
    \addplot[cyan,dashed,mark=triangle*,mark size=1,thick,mark options={solid},
                error bars/.cd, y dir=plus, y explicit, error bar style={solid}, error mark options={rotate=90,mark size=2,thick}] coordinates {
      (2000,8.624e-01) -= (0,7.343e-02) += (0,6.253e-01)  (5000,1.030e+01) -= (0,4.034e-01) += (0,1.020e+01)  (10000,7.713e+01) -= (0,3.081e+00) += (0,7.713e+01)  (20000,6.059e+02) -= (0,2.100e+01) += (0,5.992e+02)
    };
    \addplot[darkgreen,mark=triangle*,mark size=1,thick,mark options={solid,rotate=180},
                        error bars/.cd, y dir=plus, y explicit, error bar style={solid}, error mark options={rotate=90,mark size=2,thick}] coordinates {
      (2000,1.211e+00) -= (0,1.657e-02) += (0,8.758e-02)  (5000,2.136e+00) -= (0,3.715e-02) += (0,1.383e-01)  (10000,3.620e+00) -= (0,3.547e-02) += (0,1.225e-01)  (20000,7.075e+00) -= (0,7.052e-02) += (0,2.623e-01)  (50000,1.876e+01) -= (0,3.657e-01) += (0,2.645e-01)  (100000,3.406e+01) -= (0,7.541e-01) += (0,1.100e+00)
    };
    \addplot[darkgreen,mark=triangle*,dashed,mark size=1,thick,mark options={solid,rotate=180},
                        error bars/.cd, y dir=plus, y explicit, error bar style={solid}, error mark options={rotate=90,mark size=2,thick}] coordinates {
      (2000,2.960e+00) -= (0,4.079e-02) += (0,1.727e-01)  (5000,5.198e+00) -= (0,5.209e-02) += (0,3.095e-01)  (10000,9.430e+00) -= (0,6.368e-01) += (0,1.085e+00)  (20000,1.731e+01) -= (0,6.849e-01) += (0,4.203e-01)  (50000,4.532e+01) -= (0,2.223e+00) += (0,9.813e-01)  (100000,8.330e+01) -= (0,1.772e+00) += (0,1.398e+00)
    };
    \addplot[red,mark=*,mark size=1,thick,
        error bars/.cd, y dir=plus, y explicit, error bar style={solid}, error mark options={rotate=90,mark size=2,thick}] coordinates {
      (2000,2.117e-01) -= (0,6.338e-02) += (0,2.269e-01)  (5000,3.347e-01) -= (0,4.916e-03) += (0,1.156e-02)  (10000,6.563e-01) -= (0,2.155e-03) += (0,1.388e-03)  (20000,1.367e+00) -= (0,8.483e-02) += (0,2.189e-01)  (50000,3.263e+00) -= (0,3.036e-02) += (0,1.059e-01)  (100000,6.670e+00) -= (0,1.149e-01) += (0,1.777e-01)
    };
    \addplot[red,dashed,mark=*,mark size=1,thick,mark options={solid},
            error bars/.cd, y dir=plus, y explicit, error bar style={solid}, error mark options={rotate=90,mark size=2,thick}] coordinates {
      (2000,1.079e+00) -= (0,6.326e-02) += (0,7.346e-02)  (5000,1.797e+00) -= (0,9.392e-03) += (0,1.460e-02)  (10000,3.080e+00) -= (0,4.397e-02) += (0,2.475e-02)  (20000,5.507e+00) -= (0,2.892e-01) += (0,8.450e-01)  (50000,1.261e+01) -= (0,2.458e-01) += (0,3.033e-01)  (100000,2.454e+01) -= (0,2.906e-01) += (0,1.879e-01)
    };
    \addplot[red,dashdotted,mark=*,mark size=1,thick,mark options={solid},
            error bars/.cd, y dir=plus, y explicit, error bar style={solid}, error mark options={rotate=90,mark size=2,thick}] coordinates {
      (2000,1.071e+01) -= (0,2.142e-01) += (0,2.770e-01)  (5000,1.431e+01) -= (0,3.389e-01) += (0,1.863e-01)  (10000,2.098e+01) -= (0,3.538e-01) += (0,1.640e-01)  (20000,3.412e+01) -= (0,3.669e+00) += (0,1.209e+00)  (50000,7.220e+01) -= (0,6.951e+00) += (0,2.669e+00)  (100000,1.417e+02) -= (0,2.783e-01) += (0,2.127e-01)
    };
    \addplot[amber,mark=pentagon*,mark size=1.5,thick,mark options={solid},
                        error bars/.cd, y dir=both, y explicit, error bar style={solid}, error mark options={mark size=2,thick,rotate=90}] coordinates {
  (2000,9.505e-01) -= (0,1.588e-01) += (0,1.653e-01)  (5000,2.620e+00) -= (0,1.324e-01) += (0,1.520e-01)  (10000,5.391e+00) -= (0,6.486e-01) += (0,8.148e-01)  (20000,8.694e+00) -= (0,4.311e-01) += (0,4.392e-01)  (50000,2.408e+01) -= (0,9.044e-01) += (0,5.221e-01)  (100000,5.281e+01) -= (0,6.783e-01) += (0,9.232e-01)
    };
    \addplot[amber,dashed,mark=pentagon*,mark size=1.5,thick,mark options={solid},
                        error bars/.cd, y dir=both, y explicit, error bar style={solid}, error mark options={mark size=2,thick,rotate=90}] coordinates {
  (2000,6.491e+00) -= (0,2.549e-02) += (0,2.603e-02)  (5000,9.759e+00) -= (0,2.694e-01) += (0,3.160e-01)  (10000,2.193e+01) -= (0,7.984e-01) += (0,1.252e+00)  (20000,5.015e+01) -= (0,5.479e-01) += (0,5.336e-01)  (50000,1.400e+02) -= (0,5.938e+00) += (0,7.614e+00)  (100000,3.088e+02) -= (0,1.491e+01) += (0,5.426e+00)
    };
    \addplot[amber,dashdotted,mark=pentagon*,mark size=1.5,thick,mark options={solid},
                        error bars/.cd, y dir=both, y explicit, error bar style={solid}, error mark options={mark size=2,thick,rotate=90}] coordinates {
  (2000,7.792e+00) -= (0,1.039e-01) += (0,2.399e-01)  (5000,4.741e+01) -= (0,2.227e-01) += (0,3.143e-01)  (10000,1.309e+02) -= (0,4.503e+01) += (0,4.248e+01)  (20000,1.356e+02) -= (0,1.381e+01) += (0,9.738e+00)  (50000,4.045e+02) -= (0,4.187e+01) += (0,1.482e+01)  (100000,9.077e+02) -= (0,1.239e+01) += (0,1.602e+01)
    };
    \addplot[black,dotted,very thick,domain=10000:30000,samples=5] {3e-10*x^3};
    \node[pin=135:$\sim n^3$] at (axis cs:17e3,1e3) {};
    \addplot[black,dotted,very thick,domain=25000:90000,samples=5] {1.2e-6*x^2};
    \node[pin=10:$\sim n^2$] at (axis cs:4e4,1.5e3) {};
    \addplot[black,dotted,very thick,domain=8000:110000,samples=5] {4e-5*x};
    \node[pin=315:$\sim n$] at (axis cs:2.12e4,1.03e0) {};
  \end{axis}
\end{tikzpicture}
}
}
\caption{Comparison of accuracies and runtimes for spiral data sets. %
}\label{fig:compare_ev_residuals}
\end{figure}
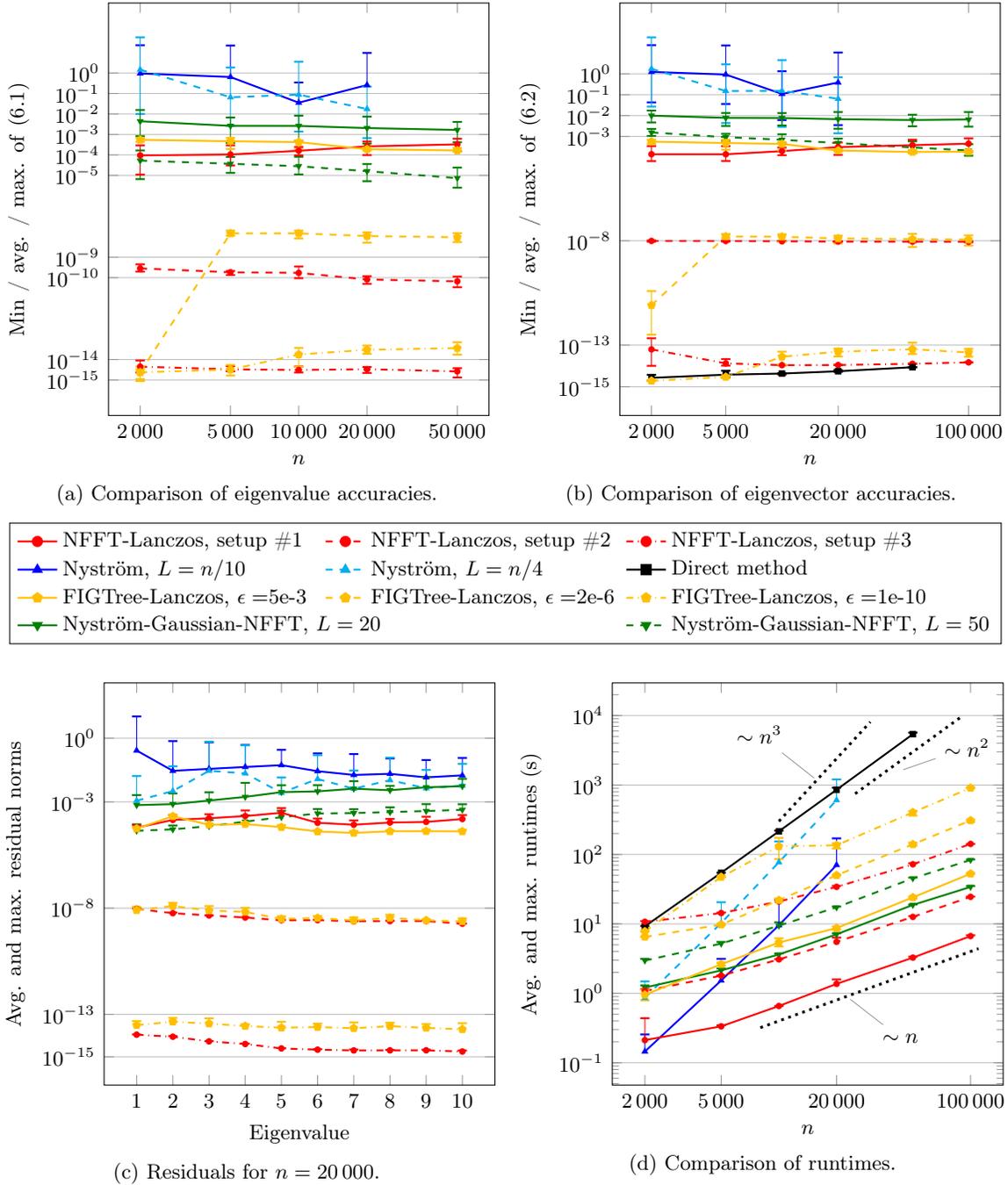

In Figure~\ref{fig:compare_ev_residuals}, we visualize the results of the test runs.
We show the minimum, average and maximum of the maximum eigenvalue errors in Figure~\ref{fig::compare_ev_residuals::eigenvalues}. For this, we first determine the maximum eigenvalue errors
\begin{equation}\label{equ::max_ev_error}
\max_{j=1,\ldots,10} \left| \lambda_j - \lambda_j^{\textnormal{(direct)}} \right|
\end{equation}
for each test run, where $\lambda_j$ denotes the $j$-th eigenvalue computed by the method under consideration and $\lambda_j^{\textnormal{(direct)}}$ the one computed by a direct method using full matrix-vector products with the matrix $\Aa$. Then, for fixed total number of data points~$n$ and fixed parameter setup, we compute the minimum, average and maximum of~\eqref{equ::max_ev_error}, where the minimum, average and maximum are computed using 5 instances of~\eqref{equ::max_ev_error} for the NFFT-based Lanczos method and $5\cdot 10$ instances of~\eqref{equ::max_ev_error} for the Nystr\"om-based methods.
We observe that the averages of the maximum eigenvalue errors~\eqref{equ::max_ev_error} are above $10^{-2}$ for the two considered parameter choices of the Nystr\"om method from Section~\ref{sec::nystrom::traditional}, even when the rank $L$ is chosen as a quarter of the matrix size~$n$. Moreover, the minima and maxima of~\eqref{equ::max_ev_error} differ distinctly from the averages. In particular, the accuracies may vary strongly across different Nystr\"om runs on an identical data set.
For the NFFT-based Lanczos method, each minimum, average and maximum of the maximum eigenvalue errors~\eqref{equ::max_ev_error} only differs slightly from one another. The maximum eigenvalue errors~\eqref{equ::max_ev_error} are around $10^{-4}$ to $10^{-3}$ for parameter setup \#1, around $10^{-10}$ to $10^{-9}$ for setup \#2, and below $10^{-14}$ for setup \#3.
For the hybrid Nystr\"om-Gaussian-NFFT method, which internally uses $2L$ many NFFT-based fast summations with parameter setup \#2, the maximum eigenvalue errors~\eqref{equ::max_ev_error} are around $10^{-3}$ to $10^{-2}$ for parameter $L=20$ and around $10^{-5}$ to $10^{-4}$ for $L=50$. This means that the observed maximum eigenvalue errors~\eqref{equ::max_ev_error} are distinctly smaller compared to the ones of the traditional Nystr\"om method,
and the errors for parameter $L=50$ are slightly smaller than the ones of the NFFT-based Lanczos method with parameter setup \#1.

In Figure~\ref{fig::compare_ev_residuals::residuals}, we depict the minimum, average and maximum of the maximum residual norms~\eqref{equ::max_res_norm} for each total number of data points~$n$. We compute these numbers by first determining the maximum residual norms
\begin{equation}\label{equ::max_res_norm}
\max_{j=1,\ldots,10} \norm{ \Aa \mathbf{v}_j - \lambda_j \mathbf{v}_j }_2
\end{equation}
for each test run,
where $\lambda_j$ denotes the $j$-th eigenvalue of $\Aa$ and $\mathbf{v}_j$ the corresponding eigenvector. Then, for fixed $n$ and fixed parameter setup, we compute the minimum, average and maximum of~\eqref{equ::max_res_norm}.
We observe that the averages of the maximum residual norms~\eqref{equ::max_res_norm} are above $10^{-1}$ for the considered parameter choices of the Nystr\"om method, even when the rank~$L$ is chosen as a quarter of the matrix size~$n$. Moreover, the minima and maxima of the maximum residual norms~\eqref{equ::max_res_norm} differ distinctly from the averages. Especially, the accuracies may vary strongly across different Nystr\"om runs on an identical data set.
For the NFFT-based Lanczos method, each minimum, average and maximum of~\eqref{equ::max_res_norm} only differs slightly from one another. The maximum residual norms~\eqref{equ::max_res_norm} are around $10^{-4}$ to $10^{-3}$ for parameter setup \#1, around $10^{-8}$ for setup~\#2, and around $10^{-15}$ to $10^{-13}$ for setup \#3.
For the hybrid Nystr\"om-Gaussian-NFFT method, maximum residual norms~\eqref{equ::max_res_norm} are around $10^{-2}$ for parameter $L=20$ and around $10^{-4}$ to $10^{-3}$ for $L=50$. In the latter case, the errors are slightly larger than the ones of the NFFT-based Lanczos method with parameter setup \#1 for $n\in\{2\,000, 5\,000, 10\,000, 20\,000\}$ data points and slightly smaller for $n\in\{50\,000, 100\,000\}$.

Additionally, in Figure~\ref{fig::compare_ev_residuals::residuals20000}, we investigate the average and maximum of the maximum residual norms~\eqref{equ::max_res_norm} for each fixed eigenvalue $\lambda_j$ for $n=20\,000$ data points. For Nystr\"om $L=n/10$, we observe that the residual norms belonging to the first eigenvalue are distinctly larger than for the remaining eigenvalues.
In general, the observed maximal residual norms~\eqref{equ::max_res_norm} vary similarly for each eigenvalue.
For the NFFT-based Lanczos method with parameter setup \#2 and \#3, the maximum residual norms~\eqref{equ::max_res_norm} of the tail eigenvalues are slightly smaller than of the leading eigenvalues,
which is not the case for the parameter setup \#1 as well as for the results of the hybrid Nystr\"om-Gaussian-NFFT method.

In Figure~\ref{fig::compare_ev_residuals::runtimes}, we show the average and maximum runtimes of the different methods and parameter choices in dependence of the total number of data points~$n$. The runtimes were determined on a computer with Intel Core i7 CPU 970 (3.20~GHz) using one thread.
We remark that the NFFT supports OpenMP, cf.\ \cite{Vo12report}, but we restricted all time measurements to 1 thread for better comparison.
We observe that the runtimes of the traditional Nystr\"om method grow approximately like $\sim n^3$, and the runtimes of the direct computation method for the eigenvalues grow approximately like $\sim n^2$.
Moreover, the slopes of the runtime graphs of the NFFT-based Lanczos method are distinctly smaller and the runtimes grow approximately like $\sim n$. Depending on the parameter choices, the NFFT-based Lanczos method is faster than the Nystr\"om method once the total number of data points $n$ is above 2\,000 -- 10\,000.
The hybrid Nystr\"om-Gaussian-NFFT method with parameter $L=20$ is slightly slower than the NFFT-based Lanczos method with setup~\#2. For the parameter $L=50$ the method is slower by a factor of approximately 2.5. In both cases, the runtimes grow approximately like $\sim n$.
The runtimes of the direct method were the highest ones in most cases. For the tests, we precomputed the diagonal entries of the matrix $\Dd^{-1/2}$ but we computed the entries of the weight matrix~$\Ww$ again for each matrix-vector multiplication with the matrix~$\Aa$. Alternatively, one could store the whole matrix~$\Aa\in\mathbb{R}^{n\times n}$ for small problem sizes~$n$ and this would have reduced the runtimes of the direct method to 1/20. However, then we would have to store at least $n(n-1)/2$ values, which would already require about 10 GB RAM for $n=50\,000$ and double precision.

For comparison, we also applied the FIGTree method from~\cite{morariu08figtree} to our testcases, and we denote the obtained results by ``FIGTree-Lanczos'' in Figure~\ref{fig:compare_ev_residuals}. The FIGTree accuracy parameter $\epsilon$ was chosen $\in\{5\cdot10^{-3},2\cdot10^{-6},10^{-10}\}$ such that the resulting residual norms~\eqref{equ::max_res_norm} in Figure~\ref{fig::compare_ev_residuals::residuals} approximately match those of the NFFT-based Lanczos method for setup~\#1,\#2,\#3. We observe that the obtained eigenvalue accuracies in Figure~\ref{fig::compare_ev_residuals::eigenvalues} are similar for $\epsilon=5\cdot10^{-3}$ and $10^{-10}$ to the ones of the NFFT-based Lanczos method for setup~\#1 and \#3, respectively. For $n\geq5\,000$ data points and FIGTree accuracy parameter $\epsilon=2\cdot10^{-6}$, we observe for our testcase that the obtained eigenvalue accuracies are lower by about two order of magnitudes compared to the NFFT-based Lanczos method with setup~\#2. When looking at the obtained runtimes, we observe that ``FIGTree-Lanczos'' requires approximately 4 times to 7 times the runtime of the corresponding NFFT-based Lanczos method with comparable eigenvector accuracy in most cases.

\subsection{Applications}
\label{sec::res::app}
In the following, we will showcase the effect of the improved accuracy on popular data science methods that utilize the graph Laplacian matrix. We will compare how the methods perform if the eigenvectors are computed with the NFFT-based Lanczos method or the traditional Nystr\"om extension.

\subsubsection{Spectral clustering}
\label{sec::res::app::clustering}

Spectral clustering is an increasingly popular technique \cite{VLu07} and we briefly illustrate the method proposed in \cite{ng2002spectral}. The basis of their algorithm is a truncated eigenapproximation 
$
\mathbf{V}_k\mathbf{D}_k\mathbf{V}_k^{T}
$
with $\mathbf{V}_k\in\R^{n\times k},$ which is an approximation based on the smallest eigenvalues and eigenvectors of the graph Laplacian. Now the rows of $\mathbf{V}_k$ are normalized to obtain a matrix $\mathbf{Y}_k$. 
The normalized rows are then divided into a fixed number of disjoint clusters by a standard k-means algorithm.

Here, we apply spectral clustering to an image segmentation problem. The original image of size $533\times 800$ is depicted in Figure~\ref{fig::RGBkmeans::original}. We construct a graph Laplacian where each pixel corresponds to a node in the graph and the distance measure is the distance between the values in all three color channels, such that each vertex $\mathbf{v}_j\in\{0,1,\dots,255\}^3$. Correspondingly, the graph Laplacian would be a dense matrix of size 426\,400 $\times$ 426\,400. We set the scaling parameter $\sigma=90$.
Figure~\ref{fig::image::ev} shows the first ten eigenvalues of the matrix~$\Aa$.

For obtaining reference results, we use the Matlab function \texttt{eigs} on the full matrix $\Aa$ computing 4 eigenvectors and this required more than 31 hours using up to 32 threads on a computer with Intel Xeon E7-4880 CPUs (2.50~GHz), using more than 500 CPU hours in total. Next, we applied the NFFT-based Lanczos method from Section~\ref{sec::lanczos} with parameters $N=16$, $m=2$, $p=2$ and $\varepsilon_\mathrm{B}=1/8$ for the eigenvector computations. We show the results in Figure~\ref{fig::RGBkmeans::nfft2} and~\ref{fig::RGBkmeans::nfft4} for $k=2$ and $k=4$ classes, respectively. The segmented images look satisfactory. The main features of the image are preserved and large areas of similar color are correctly assigned to the same cluster, while there are only small ``noisy'' areas. Compared to the segmented image from the direct computations, we have approximately 0.1~\% differences (467 out of 426,400) in the class assigments in the case of $k=4$ classes. 
For the runtimes, we measure approximately 25 seconds for the NFFT-based Lanczos method and 18 seconds for the k-means algorithm on a computer with Intel Core i7 CPU 970 (3.20~GHz) using one thread.

Additionally, we ran the Nystr\"om method 100 times with parameter $L=250$. Here the runtimes were approximately 60 seconds on average without the runtime for the clustering. We applied the k-means algorithm for $k=4$ classes, which required approximately 22 seconds on average. We observed that in 79 of the 100 test runs of Nystr\"om followed by k-means, the images appear to be very close to the ones obtained when applying \texttt{eigs} on the full matrix $\Aa$, i.e., the differences are less than 2\,\%. In Figure~\ref{fig::RGBkmeans::nystrom4_run2}, we visualize the results of a corresponding test run. 
However, in 13 of the 100 test runs, the Nystr\"om method returned eigenvectors which caused segmentation differences of more than 20\,\% with such ``noisy'' images that we consider these as ``failed'' runs. See Figure~\ref{fig::RGBkmeans::nystrom4_run59_failed} for one example with approximately 25\,\% differences.
The differences between Figure~\ref{fig::RGBkmeans::nfft4} and~\ref{fig::RGBkmeans::nystrom4_run59_failed} are shown as a black and white picture in Figure~\ref{fig::RGBkmeans::nystrom4_run59_failed_diff}.

\begin{figure}[htb!]
	\centering{
		\begin{tikzpicture}[baseline]
		\begin{axis}[
			font=\footnotesize,
			enlarge x limits=true,
			enlarge y limits=true,
			height=0.4\textwidth,
			grid=major,
			width=0.7\textwidth,
			xmin=1,xmax=10,
			ymin=0.05,ymax=1,
			xlabel={Eigenvalue number},
			ylabel={Eigenvalue},
			legend style={legend cell align=left}, legend pos=south east,
			legend columns = -1,
		]
		
		\addplot[black,mark=*,mark size=2,only marks,mark options={solid}] coordinates {
(1,1.000000) (2,0.888647) (3,0.574924) (4,0.532851) (5,0.323013) (6,0.174247) (7,0.137304) (8,0.071544) (9,0.067258) (10,0.042388)
		};
		\end{axis}
		\end{tikzpicture}
	}
	\caption{First ten eigenvalues of $\Aa$ using Gaussian weights and scaling parameter $\sigma=90$ for Figure~\ref{fig::RGBkmeans::original}.
	}\label{fig::image::ev}
\end{figure}
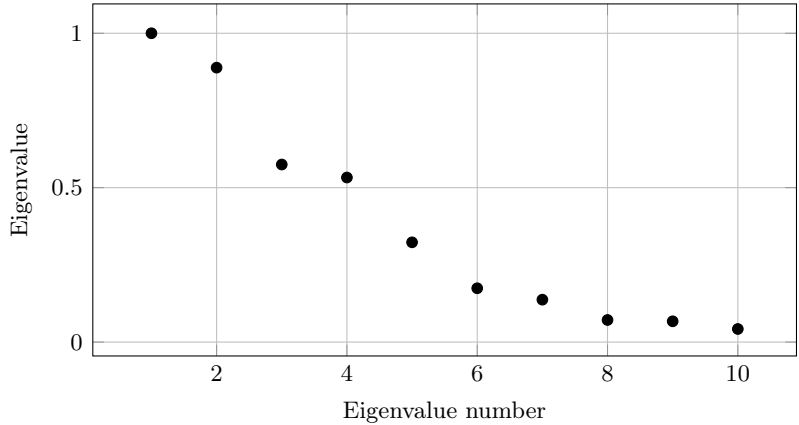

Moreover, we tested increasing the parameter $L$ to $500$. Then, the run times increased to approximately 152 seconds on average. When applying the k-means algorithm to the obtained eigenvectors, the results improved. The differences compared to the reference image segmentation are less than 2\,\% in 85 of the 100 test runs and larger than 20\,\% in 9 test runs.

\begin{figure}[htb!]
\centering{
\subfloat[Original image\protect\footnotemark]{\label{fig::RGBkmeans::original}
\includegraphics[width=7cm]{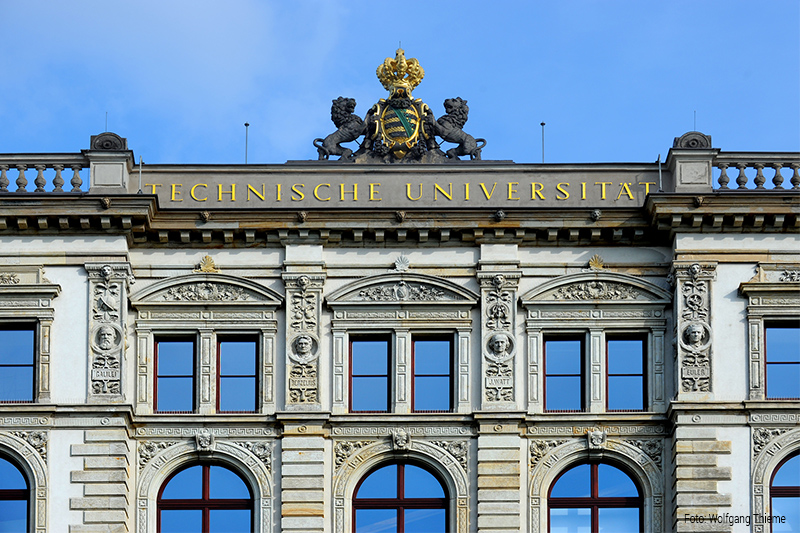}
}
\hfill
\subfloat[$k=2$ classes, NFFT-Lanczos]{\label{fig::RGBkmeans::nfft2}
\includegraphics[width=7cm]{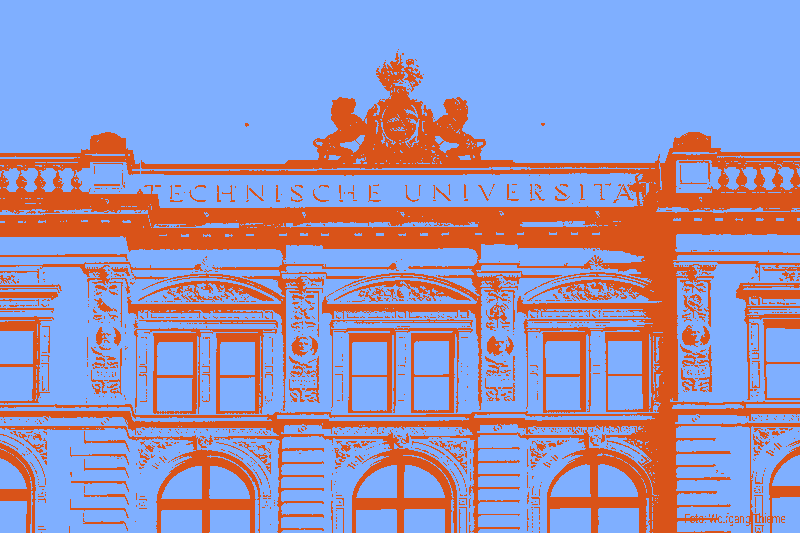}
}
\\
\subfloat[$k=4$ classes, NFFT-Lanczos]{\label{fig::RGBkmeans::nfft4}
\includegraphics[width=7cm]{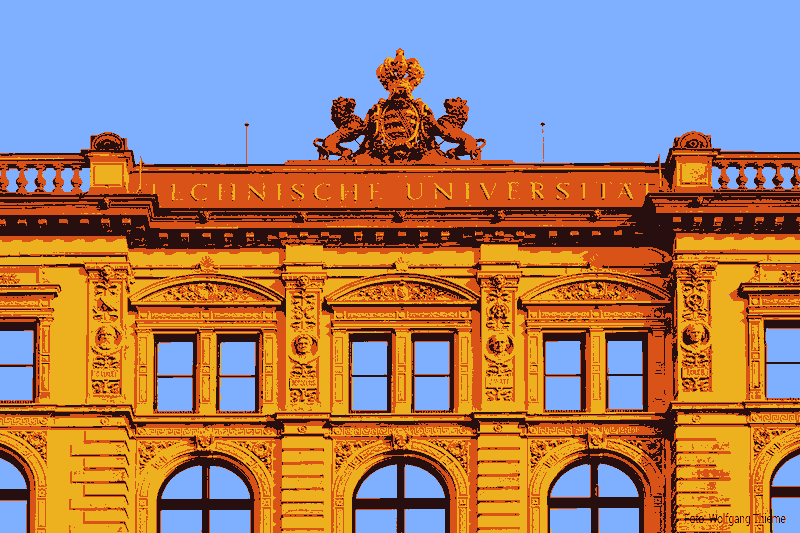}
}
\hfill
\subfloat[$k=4$ classes, Nystr\"om]{\label{fig::RGBkmeans::nystrom4_run2}
\includegraphics[width=6cm]{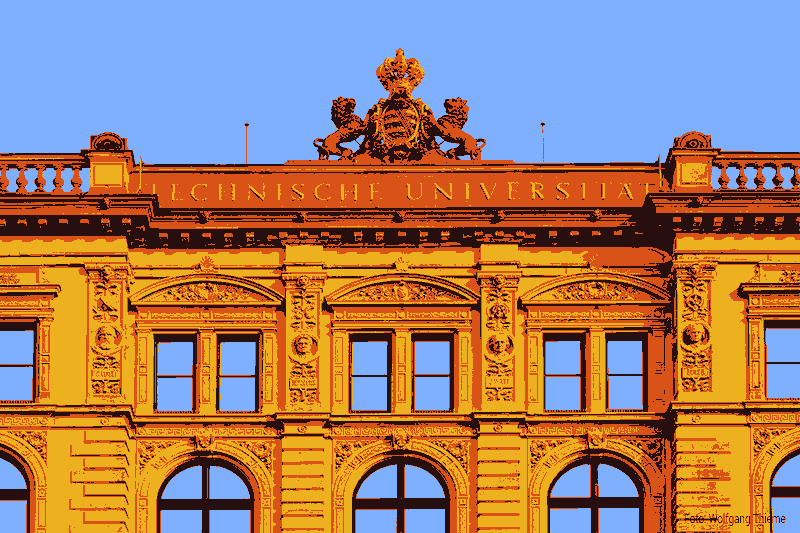}
}
\\
\subfloat[$k=4$ classes, Nystr\"om (``failed'' run)]{\label{fig::RGBkmeans::nystrom4_run59_failed}
\includegraphics[width=7cm]{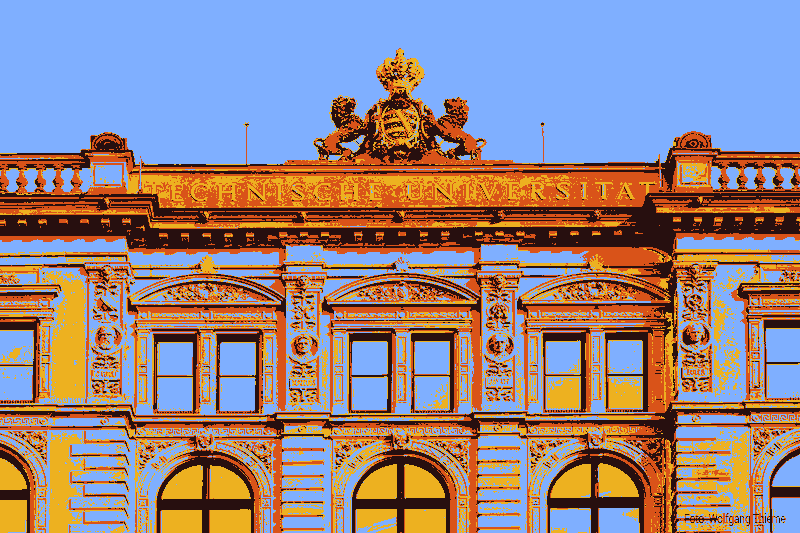}
}
\hfill
\subfloat[differences between (c) and (e)]{\label{fig::RGBkmeans::nystrom4_run59_failed_diff}
\includegraphics[width=7cm]{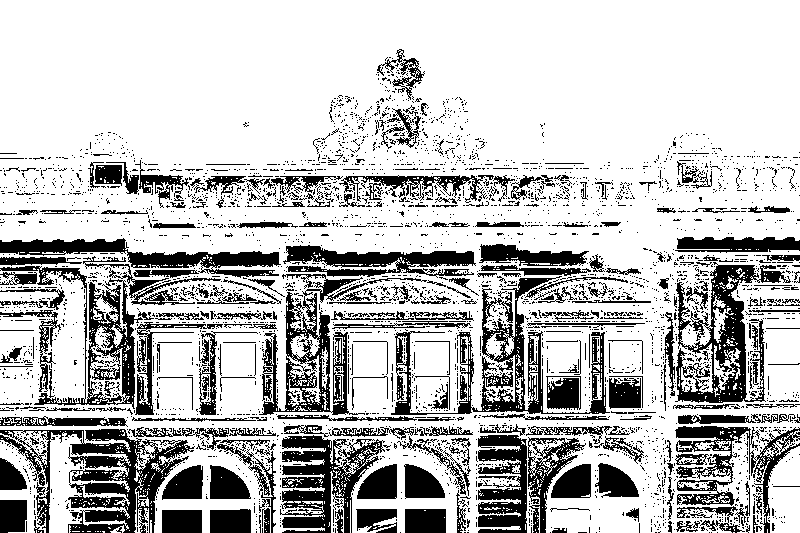}
}
}
\caption{Results of image segmentation ($533\times 800=426\,400$ pixels) via spectral clustering and k-means using the NFFT-based Lanczos method from Section~\ref{sec::lanczos} and the Nystr\"om method from Section~\ref{sec::nystrom::traditional}. ''Failed run'' in Subfigure~(e) means segmentation differences of more than 20\,\% compared to the results obtained when applying \texttt{eigs} on the full matrix $\Aa$.\label{fig::RGBkmeans}}%
\end{figure}

\subsubsection{Semi-supervised learning by a phase field method}
\label{sec::res::app::pf}

We here want to state an exemplary method that relies heavily on a number of eigenvectors of the graph Laplacian. It was proposed by Bertozzi and Flenner \cite{BerF12} and corresponds to a semi-supervised learning (SSL) problem. 
Suppose we have a graph-based dataset as before where each vertex is assigned to one of $C$ classes. A training set of $s$ random sample vertices from each class is set up. For the case of $C=2$ classes, a training vector $\mathbf f \in \R^n$ is set up with entries $-1$ for training nodes from one class, $1$ for training nodes from the other class, and $0$ for nodes that do not belong to the training data. The task of SSL is to use $\mathbf f$ to find a classification vector $\bu \in \R^n$. The sign of its entries is then used to predict each node's assigned class.

One successful approach computes $\bu$ as the end point of the trajectory described by the Allen--Cahn equation
\[ \bu: [0,\infty) \to \R^n, \qquad \bu_t=-\eps \Ll_s \bu - \frac{1}{\eps} \psi'(\bu) + \mathbf{\Omega} (\mathbf{f}-\bu), \qquad \bu(0) = \mathbf{f} \]
(see \cite{van2014mean,luo2017convergence} for details). 
Here $\psi(u)=(u^2-1)^2$ is the double-well potential, which we understand to be applied component-wise, and $\mathbf{\Omega}$ denotes a diagonal matrix with entries $\Omega_{ii} = \omega_0 > 0$ if vertex $i$ belongs to the training data and $\Omega_{ii} = 0$ otherwise.
To discretize this ODE we will not introduce an index for the temporal discretization but rather assume that all values $\bu$ are evaluated at the new time-point whereas $\bar{\bu}$ indicates the previous time-point.  We then obtain 
\begin{equation*}
\frac{\bu-\bar{\bu}}{\tau}+\eps \Ll_s \bu+c\bu =-\frac{1}{\eps}\psi'(\bar{\bu})+c\bar{\bu}+\mathbf{\Omega}(\mathbf{f}-\bar{\bu}),
\end{equation*}
where $\bu$ is a vector defined on the graph on which we base the final classification decision. Here, $c>0$ is a positive parameter for the convexity splitting technique \cite{BerF12}.
For a more detailed discussion of how to set these parameters we refer to \cite{BerF12,BosKS17}.
We now use the $k$ computed eigenvalues and eigenvectors $(\lambda_j,\bv_j)$ of $\Ll_s$ such that we can write $\bu=\sum_{j=1}^{k}{u}_j\bv_j$
and from this we get
\begin{align*}
&\frac{{u}_j - \bar{{u}}_j}{\tau} + \eps \lambda_j {u}_j + c{u}_j = 
-\frac{1}{\eps} \vb_j^T \psi'(\bar{\bu}) + c \bar{{u}}_j + \vb_j^T \mathbf{\Omega} (\mathbf{f} - \bar{\bu}).
\end{align*}
This equation can be solved to obtain the new coefficients $u_j$ from the old coefficients $\bar{u}_j$. After a sufficient number of time steps, $\bu$ will converge against a stable solution.

We apply this method to the same spiral data set as seen in Section~\ref{sec::res::ev}, again with $\sigma=3.5$ but this time only with $n=100\,000$. 
\footnotetext{Image source: TU Chemnitz/Wolfgang Thieme}
The data points have been generated by a multivariate normal distribution around five center points, and the true label of each vertex has been set to the center point that is closest to it.
We computed the eigenvectors to the $k=5$ smallest eigenvalues of the Laplacian; once by the NFFT-based Lanczos method with $n=32$, $m=4$, and $\eps_B=0$, and once with the traditional Nystr\"om method with $L=1\,000$ where only $5$ columns of $\Vv_L$ are used. We then applied the described method with $\tau=0.1$, $\eps=10$, $\omega_0 = 10\,000$, and $c = \frac{2}{\eps} + \omega_0$. The iteration terminated if the squared relative change in $\bu$ was less than 1e-10. We repeat this process for 50 instances of the spiral dataset and sample sizes $s \in \{1,2,3,4,5,7,10\}$.

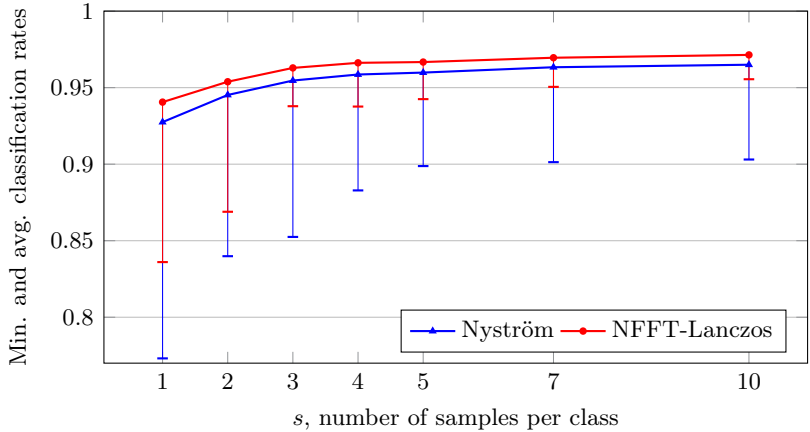
\begin{figure}[htb!]
	\centering{
		\begin{tikzpicture}[baseline]
		\begin{axis}[
			font=\footnotesize,
			enlarge x limits=true,
			height=0.4\textwidth,
			grid=major,
			xmajorgrids=false,
			width=0.7\textwidth,
			xtick={1,2,3,4,5,7,10},
			ytick={1,0.95,0.9,0.85, 0.8},
			xmin=1,xmax=10,
			ymin=0.77,ymax=1,
			xlabel={$s$, number of samples per class},
			ylabel={Min. and avg. classification rates},
			legend style={legend cell align=left}, legend pos=south east,
			legend columns = -1,
		]
		
		\addplot[blue,mark=triangle*,mark size=1,thick,mark options={solid},
		error bars/.cd, y dir=minus, y explicit, error bar style={solid}, error mark options={rotate=90,mark size=2,thick}] coordinates {
			(1,9.2749e-01) -= (0, 1.5437e-01) += (0, 4.8970e-02)
			(2,9.4521e-01) -= (0, 1.0533e-01) += (0, 3.8968e-02)
			(3,9.5463e-01) -= (0, 1.0216e-01) += (0, 2.8400e-02)
			(4,9.5858e-01) -= (0, 7.5722e-02) += (0, 2.4458e-02)
			(5,9.5984e-01) -= (0, 6.1083e-02) += (0, 2.4347e-02)
			(7,9.6333e-01) -= (0, 6.2005e-02) += (0, 2.0835e-02)
			(10,9.6497e-01) -= (0, 6.1918e-02) += (0, 1.8962e-02)
		};
		\addlegendentry{Nystr\"om} %
		
		\addplot[red,mark=*,mark size=1,thick,mark options={solid},
		error bars/.cd, y dir=minus, y explicit, error bar style={solid}, error mark options={rotate=90,mark size=2,thick}] coordinates {
			(1,9.4055e-01) -= (0, 1.0452e-01) += (0, 3.6575e-02)
			(2,9.5386e-01) -= (0, 8.4955e-02) += (0, 2.8775e-02)
			(3,9.6289e-01) -= (0, 2.5016e-02) += (0, 1.9344e-02)
			(4,9.6621e-01) -= (0, 2.8568e-02) += (0, 1.6472e-02)
			(5,9.6672e-01) -= (0, 2.4226e-02) += (0, 1.6684e-02)
			(7,9.6954e-01) -= (0, 1.9032e-02) += (0, 1.3678e-02)
			(10,9.7139e-01) -= (0, 1.5885e-02) += (0, 1.1715e-02)
			
		};
		\addlegendentry{NFFT-Lanczos} %
		\end{axis}
		\end{tikzpicture}
	}
	\caption{Comparison of average classification rates with the phase field method on relabeled spiral data sets. %
	}\label{fig::SSLPF}
\end{figure}

Figure~\ref{fig::SSLPF} depicts the average accuracy results. We conclude that in this example, the increased eigenvector quality achieved by the NFFT-based method yields an average accuracy boost of approximately 0.5 to 1.5 percentage points, as well as the worst result being significantly less bad. On a computer with Intel Core i7 CPU 4770 (3.40 GHz), the runtimes were approximately 8 seconds for the NFFT-based Lanczos method, 27 seconds for the Nystr\"om method, and less than a second for the solution of the Allen--Cahn equation, which almost always converged after only three time steps.

\subsubsection{Semi-supervised learning by a kernel method}
\label{sec::res::app::kernel}
In addition to the phase field method, we employ a second semi-supervised learning technique used in \cite{zhou2004learning,hein2013total} for SSL problems with only two classes. Based on a training vector $\mathbf{f}$ holding 1, -1, or 0 just as in the previous section, a similar $\bu$ is obtained by minimizing the function 
\begin{equation}
\label{eq::ssl}
\argmin_{\bu \in \R^n} \;
\frac{1}{2}\norm{\bu-\mathbf{f}}_2^{2}+\frac{\beta}{2} \bu^{T} \Ll_s\bu,
\end{equation}
where $\beta$ can be understood as a regularization parameter.
For the solution of this minimization problem, we only have to solve the equation
\begin{equation}
\label{eq::ssl::system}
\left(\mathbf{I}+\beta \Ll_s\right)\bu=\mathbf{f},
\end{equation}
where $\mathbf{I}$ is the identity matrix. Similar systems arise naturally in scattered data interpolation~\cite{iske2017hierarchical}.
We run numerical tests using the \texttt{crescentfullmoon.m}\footnote{\url{https://www.mathworks.com/matlabcentral/fileexchange/41459-6-functions-for-generating-artificial-datasets}} data set with $n=100\,000$ data points and parameters \texttt{r1}=5, \texttt{r2}=5, \texttt{r3}=8. As illustrated in Figure~\ref{fig::illustration_datasets::crescentfullmoon}, the set is divided into two classes of points in the full moon and the crescent, distributed in a 1-to-3 ratio.
We generate 5 random instances of the data set, and for each instance we run 10 repetitions with randomly chosen training data, where we consider $s \in\{1,2,5,10,25\}$ known samples per class. For the adjacency matrix $\Ww$, we set the scaling parameter $\sigma=0.1$. The tests are run with regularization parameter $\beta\in\{10^3,3\cdot 10^3,10^4,3\cdot 10^4,10^5\}$. We solve each system~\eqref{eq::ssl::system} using the CG algorithm with tolerance parameter $10^{-4}$ and a maximum number of 1\,000 iterations. For the fast matrix-vector multiplications with the matrix~$\Ll_s$, we use the NFFT-based fast summation in Algorithm~\ref{alg::nfft_fastsum} with parameters $N=512$, $m=3$, $\varepsilon_\mathrm{B}=0$.

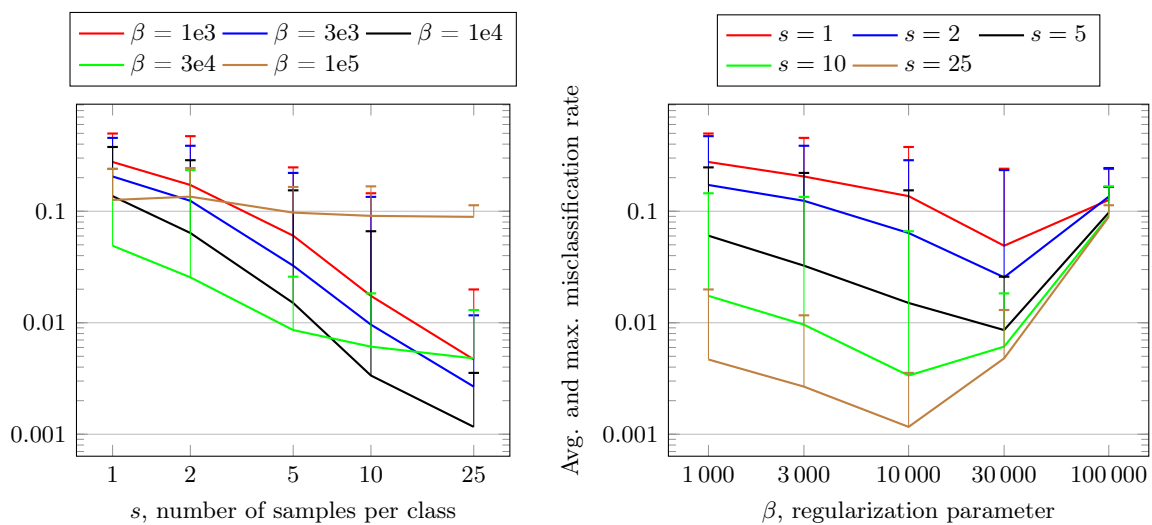
\begin{figure}[htb!]
\centering{
\begin{tikzpicture}[baseline]
  \begin{axis}[font=\footnotesize,enlargelimits=true,xmin=1,xmax=25,height=0.4\textwidth, width=0.47\textwidth, grid=major, xlabel={$s$, number of samples per class}, %
    xtick={1,2,5,10,25},xticklabel={\pgfkeys{/pgf/fpu=true}\pgfmathparse{exp(\tick)}\pgfmathprintnumber[1000 sep={\,},fixed relative, precision=3]{\pgfmathresult}\pgfkeys{/pgf/fpu=false}},
    yticklabel={\pgfkeys{/pgf/fpu=true}\pgfmathparse{exp(\tick)}\pgfmathprintnumber[1000 sep={\,},fixed relative, precision=3]{\pgfmathresult}\pgfkeys{/pgf/fpu=false}},
    legend style={legend cell align=left, align=left, at={(0.5,1.05)}, anchor=south},%
    legend columns=3,
    xmode=log,ymode=log,
    xmajorgrids=false,
  ]
  \addplot[red,thick,error bars/.cd, y dir=plus, y explicit, error bar style={solid}, error mark options={rotate=90,mark size=2,thick}] coordinates {
    (1,2.767e-01) -= (0,1.631e-01) += (0,2.213e-01)  (2,1.721e-01) -= (0,1.608e-01) += (0,2.995e-01)  (5,6.037e-02) -= (0,5.913e-02) += (0,1.871e-01)  (10,1.742e-02) -= (0,1.545e-02) += (0,1.276e-01)  (25,4.682e-03) -= (0,2.972e-03) += (0,1.521e-02)
  };
  \addlegendentry{$\beta$ = 1e3}
  \addplot[blue,thick,error bars/.cd, y dir=plus, y explicit, error bar style={solid}, error mark options={rotate=90,mark size=2,thick}] coordinates {
    (1,2.051e-01) -= (0,1.893e-01) += (0,2.500e-01)  (2,1.241e-01) -= (0,1.209e-01) += (0,2.626e-01)  (5,3.251e-02) -= (0,3.147e-02) += (0,1.882e-01)  (10,9.609e-03) -= (0,8.389e-03) += (0,1.250e-01)  (25,2.673e-03) -= (0,1.483e-03) += (0,8.987e-03)
  };
  \addlegendentry{$\beta$ = 3e3}
  \addplot[black,thick,error bars/.cd, y dir=plus, y explicit, error bar style={solid}, error mark options={rotate=90,mark size=2,thick}] coordinates {
    (1,1.367e-01) -= (0,1.351e-01) += (0,2.405e-01)  (2,6.384e-02) -= (0,6.285e-02) += (0,2.232e-01)  (5,1.505e-02) -= (0,1.436e-02) += (0,1.390e-01)  (10,3.361e-03) -= (0,2.541e-03) += (0,6.283e-02)  (25,1.164e-03) -= (0,5.542e-04) += (0,2.386e-03)
  };
  \addlegendentry{$\beta$ = 1e4}
  \addplot[green,thick,error bars/.cd, y dir=plus, y explicit, error bar style={solid}, error mark options={rotate=90,mark size=2,thick}] coordinates {
    (1,4.906e-02) -= (0,4.869e-02) += (0,1.917e-01)  (2,2.556e-02) -= (0,2.511e-02) += (0,2.084e-01)  (5,8.598e-03) -= (0,8.108e-03) += (0,1.729e-02)  (10,6.104e-03) -= (0,5.614e-03) += (0,1.224e-02)  (25,4.790e-03) -= (0,4.520e-03) += (0,8.220e-03)
  };
  \addlegendentry{$\beta$ = 3e4}
  \addplot[brown,thick,error bars/.cd, y dir=plus, y explicit, error bar style={solid}, error mark options={rotate=90,mark size=2,thick}] coordinates {
    (1,1.263e-01) -= (0,1.156e-01) += (0,1.131e-01)  (2,1.352e-01) -= (0,9.132e-02) += (0,1.090e-01)  (5,9.723e-02) -= (0,6.520e-02) += (0,6.798e-02)  (10,9.088e-02) -= (0,3.228e-02) += (0,7.634e-02)  (25,8.908e-02) -= (0,2.505e-02) += (0,2.412e-02)
  };
  \addlegendentry{$\beta$ = 1e5}
  \end{axis}
\end{tikzpicture}
\hfill
\begin{tikzpicture}[baseline]
  \begin{axis}[font=\footnotesize,enlargelimits=true,xmin=1e3,xmax=1e5,height=0.4\textwidth, width=0.51\textwidth, grid=major, xlabel={$\beta$, regularization parameter}, ylabel={Avg. and max. misclassification rate},
    xtick={1e3,3e3,1e4,3e4,1e5},xticklabel={\pgfkeys{/pgf/fpu=true}\pgfmathparse{exp(\tick)}\pgfmathprintnumber[1000 sep={\,},fixed relative, precision=3]{\pgfmathresult}\pgfkeys{/pgf/fpu=false}},
    yticklabel={\pgfkeys{/pgf/fpu=true}\pgfmathparse{exp(\tick)}\pgfmathprintnumber[1000 sep={\,},fixed relative, precision=3]{\pgfmathresult}\pgfkeys{/pgf/fpu=false}},
    legend style={legend cell align=left, align=left, at={(0.5,1.05)}, anchor=south},%
    legend columns=3,
    xmode=log,ymode=log,
    xmajorgrids=false,
  ]
  \addplot[red,thick,error bars/.cd, y dir=plus, y explicit, error bar style={solid}, error mark options={rotate=90,mark size=2,thick}] coordinates {
  (1.000e+03,2.767e-01) -= (0,1.631e-01) += (0,2.213e-01)  (3.000e+03,2.051e-01) -= (0,1.893e-01) += (0,2.500e-01)  (1.000e+04,1.367e-01) -= (0,1.351e-01) += (0,2.405e-01)  (3.000e+04,4.906e-02) -= (0,4.869e-02) += (0,1.917e-01)  (1.000e+05,1.263e-01) -= (0,1.156e-01) += (0,1.131e-01)
  };
  \addlegendentry{$s=1$}
  \addplot[blue,thick,error bars/.cd, y dir=plus, y explicit, error bar style={solid}, error mark options={rotate=90,mark size=2,thick}] coordinates {
  (1.000e+03,1.721e-01) -= (0,1.608e-01) += (0,2.995e-01)  (3.000e+03,1.241e-01) -= (0,1.209e-01) += (0,2.626e-01)  (1.000e+04,6.384e-02) -= (0,6.285e-02) += (0,2.232e-01)  (3.000e+04,2.556e-02) -= (0,2.511e-02) += (0,2.084e-01)  (1.000e+05,1.352e-01) -= (0,9.132e-02) += (0,1.090e-01)
  };
  \addlegendentry{$s=2$}
  \addplot[black,thick,error bars/.cd, y dir=plus, y explicit, error bar style={solid}, error mark options={rotate=90,mark size=2,thick}] coordinates {
  (1.000e+03,6.037e-02) -= (0,5.913e-02) += (0,1.871e-01)  (3.000e+03,3.251e-02) -= (0,3.147e-02) += (0,1.882e-01)  (1.000e+04,1.505e-02) -= (0,1.436e-02) += (0,1.390e-01)  (3.000e+04,8.598e-03) -= (0,8.108e-03) += (0,1.729e-02)  (1.000e+05,9.723e-02) -= (0,6.520e-02) += (0,6.798e-02)
  };
  \addlegendentry{$s=5$}
  \addplot[green,thick,error bars/.cd, y dir=plus, y explicit, error bar style={solid}, error mark options={rotate=90,mark size=2,thick}] coordinates {
  (1.000e+03,1.742e-02) -= (0,1.545e-02) += (0,1.276e-01)  (3.000e+03,9.609e-03) -= (0,8.389e-03) += (0,1.250e-01)  (1.000e+04,3.361e-03) -= (0,2.541e-03) += (0,6.283e-02)  (3.000e+04,6.104e-03) -= (0,5.614e-03) += (0,1.224e-02)  (1.000e+05,9.088e-02) -= (0,3.228e-02) += (0,7.634e-02)
  };
  \addlegendentry{$s=10$}
  \addplot[brown,thick,error bars/.cd, y dir=plus, y explicit, error bar style={solid}, error mark options={rotate=90,mark size=2,thick}] coordinates {
  (1.000e+03,4.682e-03) -= (0,2.972e-03) += (0,1.521e-02)  (3.000e+03,2.673e-03) -= (0,1.483e-03) += (0,8.987e-03)  (1.000e+04,1.164e-03) -= (0,5.542e-04) += (0,2.386e-03)  (3.000e+04,4.790e-03) -= (0,4.520e-03) += (0,8.220e-03)  (1.000e+05,8.908e-02) -= (0,2.505e-02) += (0,2.412e-02)
  };
  \addlegendentry{$s=25$}
  \end{axis}
\end{tikzpicture}
}
\caption{Misclassification rate solving~\eqref{eq::ssl::system} using the CG algorithm and Algorithm~\ref{alg::nfft_fastsum} for the \texttt{crescentfullmoon.m} data set with $n=100\,000$ data points.
}\label{fig:test_SSL_kernel_crescentfullmoon_fs}
\end{figure}

In Figure~\ref{fig:test_SSL_kernel_crescentfullmoon_fs}, we visualize the average and maximum misclassification rate of the $5\cdot 10$ test runs for each fixed $s$ and $\beta$.
In the left plot, we show the misclassification rate in dependence of the number of samples~$s$ per class for the different regularization parameters~$\beta$. We observe in general that the misclassification rates decrease for increasing~$s$. The lowest rate is achieved for $s=25$ samples per class and $\beta=10^4$, where the average and maximum misclassification rate are 0.0012 and 0.0036, respectively.
In the right plot, we depict the misclassification rate in dependence of the regularization parameter~$\beta$ for fixed number of samples~$s$ per class. For $s\in\{1,2,5\}$, the average misclassification rates decline for increasing $\beta$ until $\beta=3\cdot 10^4$ and grow again for $\beta=10^5$. For $s\in\{10,25\}$, the average misclassification rates decline for increasing $\beta$ until $\beta=10^4$ and grow again afterwards.
We remark that in all test runs, the maximum number of CG iterations was 536 and the maximum runtime for solving \eqref{eq::ssl::system} was approximately 151\,seconds on a computer with Intel Core i7 CPU 970 (3.20~GHz) using one thread.

Additionally, we used the NFFT-based Lanczos method from Section~\ref{sec::lanczos} in order to approximate the matrix $\Aa:=\Dd^{-1/2}\Ww\Dd^{-1/2}$ by a truncated eigenapproximation 
$
\mathbf{V}_k\mathbf{D}_k\mathbf{V}_k^{T}
$
with $\mathbf{V}_k\in\R^{n,k}$ and this allows for computing the matrix-vector products in ~\eqref{eq::ssl::system} in a fast way for fixed small~$k$. Using $k=10$ eigenvalues and eigenvectors, we achieve similar results as those shown in Figure~\ref{fig:test_SSL_kernel_crescentfullmoon_fs}. The computation of the eigenapproximation required up to 6 minutes on a computer with Intel Core i7 CPU 970 (3.20~GHz) using one thread. The maximum runtime for solving \eqref{eq::ssl::system} was approximately 0.15\,seconds.

Alternatively, we applied the Nystr\"om method from Section~\ref{sec::nystrom::traditional} with parameter $L=5\,000$ to obtain a truncated eigenapproximation, where the corresponding computation required more than 3 hours for each eigenapproximation. 
However, the eigenvalues were not computed correctly in our tests. This was due to the matrix block $\Ww_{XX}$ in Equation \eqref{eq:nystromapprox} being ill-conditioned. Consequently the CG method aborted in the first iteration and the output could not be used for classification.

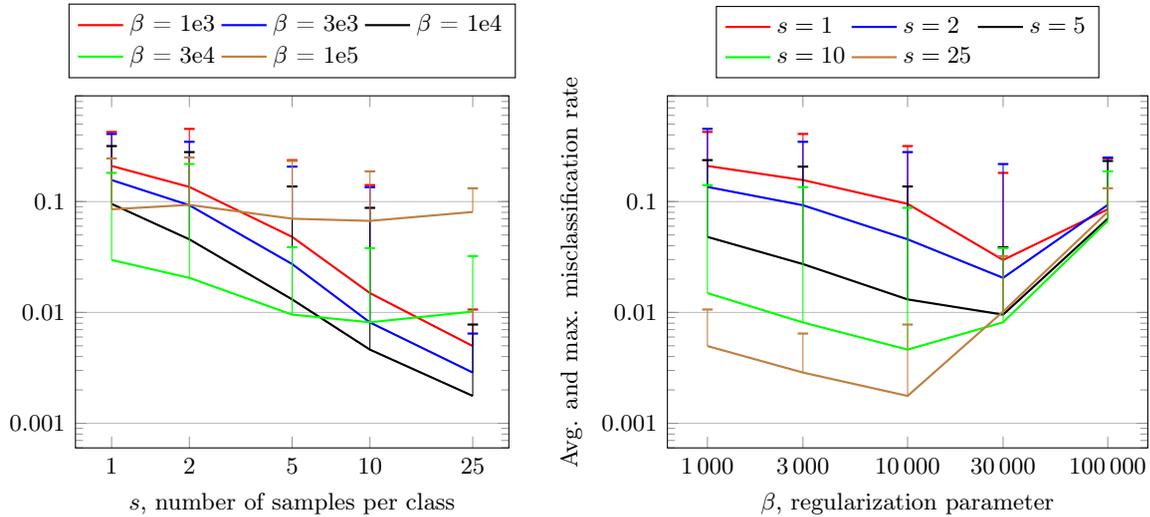
\begin{figure}[htb!]
\centering{
\begin{tikzpicture}[baseline]
  \begin{axis}[font=\footnotesize,enlarge x limits=true,enlarge y limits=false,xmin=1,xmax=25,height=0.4\textwidth, width=0.47\textwidth, grid=major, xlabel={$s$, number of samples per class}, %
    xtick={1,2,5,10,25},xticklabel={\pgfkeys{/pgf/fpu=true}\pgfmathparse{exp(\tick)}\pgfmathprintnumber[1000 sep={\,},fixed relative, precision=3]{\pgfmathresult}\pgfkeys{/pgf/fpu=false}},
    yticklabel={\pgfkeys{/pgf/fpu=true}\pgfmathparse{exp(\tick)}\pgfmathprintnumber[1000 sep={\,},fixed relative, precision=3]{\pgfmathresult}\pgfkeys{/pgf/fpu=false}},
    legend style={legend cell align=left, align=left, at={(0.5,1.05)}, anchor=south},%
    legend columns=3,
    xmode=log,ymode=log,
    xmajorgrids=false,
    ymin=6e-4,ymax=0.9,
  ]
  \addplot[red,thick,error bars/.cd, y dir=plus, y explicit, error bar style={solid}, error mark options={rotate=90,mark size=2,thick}] coordinates {
  (1,2.097e-01) -= (0,1.805e-01) += (0,2.155e-01)  (2,1.356e-01) -= (0,1.258e-01) += (0,3.173e-01)  (5,4.803e-02) -= (0,4.668e-02) += (0,1.887e-01)  (10,1.495e-02) -= (0,1.200e-02) += (0,1.259e-01)  (25,4.975e-03) -= (0,3.145e-03) += (0,5.655e-03)
  };
  \addlegendentry{$\beta$ = 1e3}
  \addplot[blue,thick,error bars/.cd, y dir=plus, y explicit, error bar style={solid}, error mark options={rotate=90,mark size=2,thick}] coordinates {
  (1,1.565e-01) -= (0,1.522e-01) += (0,2.511e-01)  (2,9.273e-02) -= (0,8.963e-02) += (0,2.536e-01)  (5,2.732e-02) -= (0,2.609e-02) += (0,1.794e-01)  (10,8.124e-03) -= (0,6.394e-03) += (0,1.268e-01)  (25,2.869e-03) -= (0,1.559e-03) += (0,3.581e-03)
  };
  \addlegendentry{$\beta$ = 3e3}
  \addplot[black,thick,error bars/.cd, y dir=plus, y explicit, error bar style={solid}, error mark options={rotate=90,mark size=2,thick}] coordinates {
  (1,9.539e-02) -= (0,9.433e-02) += (0,2.216e-01)  (2,4.568e-02) -= (0,4.468e-02) += (0,2.340e-01)  (5,1.311e-02) -= (0,1.212e-02) += (0,1.237e-01)  (10,4.620e-03) -= (0,3.670e-03) += (0,8.319e-02)  (25,1.766e-03) -= (0,1.086e-03) += (0,6.014e-03)
  };
  \addlegendentry{$\beta$ = 1e4}
  \addplot[green,thick,error bars/.cd, y dir=plus, y explicit, error bar style={solid}, error mark options={rotate=90,mark size=2,thick}] coordinates {
  (1,2.974e-02) -= (0,2.876e-02) += (0,1.518e-01)  (2,2.052e-02) -= (0,1.963e-02) += (0,1.977e-01)  (5,9.557e-03) -= (0,9.027e-03) += (0,2.929e-02)  (10,8.152e-03) -= (0,7.722e-03) += (0,2.988e-02)  (25,1.018e-02) -= (0,9.566e-03) += (0,2.206e-02)
  };
  \addlegendentry{$\beta$ = 3e4}
  \addplot[brown,thick,error bars/.cd, y dir=plus, y explicit, error bar style={solid}, error mark options={rotate=90,mark size=2,thick}] coordinates {
  (1,8.526e-02) -= (0,8.137e-02) += (0,1.599e-01)  (2,9.345e-02) -= (0,8.941e-02) += (0,1.562e-01)  (5,7.012e-02) -= (0,6.698e-02) += (0,1.625e-01)  (10,6.701e-02) -= (0,5.454e-02) += (0,1.204e-01)  (25,8.063e-02) -= (0,4.388e-02) += (0,5.132e-02)
  };
  \addlegendentry{$\beta$ = 1e5}
  \end{axis}
\end{tikzpicture}
\hfill
\begin{tikzpicture}[baseline]
  \begin{axis}[font=\footnotesize,enlarge x limits=true,enlarge y limits=false,xmin=1e3,xmax=1e5,height=0.4\textwidth, width=0.51\textwidth, grid=major, xlabel={$\beta$, regularization parameter}, ylabel={Avg. and max. misclassification rate},
    xtick={1e3,3e3,1e4,3e4,1e5},xticklabel={\pgfkeys{/pgf/fpu=true}\pgfmathparse{exp(\tick)}\pgfmathprintnumber[1000 sep={\,},fixed relative, precision=3]{\pgfmathresult}\pgfkeys{/pgf/fpu=false}},
    yticklabel={\pgfkeys{/pgf/fpu=true}\pgfmathparse{exp(\tick)}\pgfmathprintnumber[1000 sep={\,},fixed relative, precision=3]{\pgfmathresult}\pgfkeys{/pgf/fpu=false}},
    legend style={legend cell align=left, align=left, at={(0.5,1.05)}, anchor=south},%
    legend columns=3,
    xmode=log,ymode=log,
    xmajorgrids=false,
    ymin=6e-4,ymax=0.9,
  ]
  \addplot[red,thick,error bars/.cd, y dir=plus, y explicit, error bar style={solid}, error mark options={rotate=90,mark size=2,thick}] coordinates {
  (1.000e+03,2.097e-01) -= (0,1.805e-01) += (0,2.155e-01)  (3.000e+03,1.565e-01) -= (0,1.522e-01) += (0,2.511e-01)  (1.000e+04,9.539e-02) -= (0,9.433e-02) += (0,2.216e-01)  (3.000e+04,2.974e-02) -= (0,2.876e-02) += (0,1.518e-01)  (1.000e+05,8.526e-02) -= (0,8.137e-02) += (0,1.599e-01)
  };
  \addlegendentry{$s=1$}
  \addplot[blue,thick,error bars/.cd, y dir=plus, y explicit, error bar style={solid}, error mark options={rotate=90,mark size=2,thick}] coordinates {
  (1.000e+03,1.356e-01) -= (0,1.258e-01) += (0,3.173e-01)  (3.000e+03,9.273e-02) -= (0,8.963e-02) += (0,2.536e-01)  (1.000e+04,4.568e-02) -= (0,4.468e-02) += (0,2.340e-01)  (3.000e+04,2.052e-02) -= (0,1.963e-02) += (0,1.977e-01)  (1.000e+05,9.345e-02) -= (0,8.941e-02) += (0,1.562e-01)
  };
  \addlegendentry{$s=2$}
  \addplot[black,thick,error bars/.cd, y dir=plus, y explicit, error bar style={solid}, error mark options={rotate=90,mark size=2,thick}] coordinates {
  (1.000e+03,4.803e-02) -= (0,4.668e-02) += (0,1.887e-01)  (3.000e+03,2.732e-02) -= (0,2.609e-02) += (0,1.794e-01)  (1.000e+04,1.311e-02) -= (0,1.212e-02) += (0,1.237e-01)  (3.000e+04,9.557e-03) -= (0,9.027e-03) += (0,2.929e-02)  (1.000e+05,7.012e-02) -= (0,6.698e-02) += (0,1.625e-01)
  };
  \addlegendentry{$s=5$}
  \addplot[green,thick,error bars/.cd, y dir=plus, y explicit, error bar style={solid}, error mark options={rotate=90,mark size=2,thick}] coordinates {
  (1.000e+03,1.495e-02) -= (0,1.200e-02) += (0,1.259e-01)  (3.000e+03,8.124e-03) -= (0,6.394e-03) += (0,1.268e-01)  (1.000e+04,4.620e-03) -= (0,3.670e-03) += (0,8.319e-02)  (3.000e+04,8.152e-03) -= (0,7.722e-03) += (0,2.988e-02)  (1.000e+05,6.701e-02) -= (0,5.454e-02) += (0,1.204e-01)
  };
  \addlegendentry{$s=10$}
  \addplot[brown,thick,error bars/.cd, y dir=plus, y explicit, error bar style={solid}, error mark options={rotate=90,mark size=2,thick}] coordinates {
  (1.000e+03,4.975e-03) -= (0,3.145e-03) += (0,5.655e-03)  (3.000e+03,2.869e-03) -= (0,1.559e-03) += (0,3.581e-03)  (1.000e+04,1.766e-03) -= (0,1.086e-03) += (0,6.014e-03)  (3.000e+04,1.018e-02) -= (0,9.566e-03) += (0,2.206e-02)  (1.000e+05,8.063e-02) -= (0,4.388e-02) += (0,5.132e-02)
  };
  \addlegendentry{$s=25$}
  \end{axis}
\end{tikzpicture}
}
\caption{Misclassification rate solving~\eqref{eq::ssl::system} using the CG algorithm and Algorithm~\ref{alg::nfft_fastsum} for the \texttt{crescentfullmoon.m} data set with $n=100\,000$ data points and Laplacian RBF kernel~\eqref{eq::W_laplacian_rbf}.
}\label{fig:test_SSL_kernel_crescentfullmoon_fs:laplacian_kernel}
\end{figure}

In order to illustrate the flexibility of the NFFT-based fast summation, we also apply Algorithm~\ref{alg::nfft_fastsum} to a non-Gaussian weight function $w$ in~\eqref{eq::W_gaussian}. Here, we consider the ``Laplacian RBF kernel'' $K(\mathbf{y}) := \exp (-\norm{\mathbf{y}}/\sigma)$, such that the weight matrix is constructed as
\begin{equation}\label{eq::W_laplacian_rbf}
W_{ji} = 
w(v_j,v_i)=
\begin{cases}
\exp(-\norm{\vb_j-\vb_i}/\sigma)&\textnormal{ if } 
j \neq i, \\
0 & \textnormal{ otherwise}.
\end{cases}
\end{equation}
In our numerical tests, we set the shape parameter $\sigma=0.05$ and we visualize the test results in Figure~\ref{fig:test_SSL_kernel_crescentfullmoon_fs:laplacian_kernel}.
We observe that the obtained misclassification rates are similar to the ones in Figure~\ref{fig:test_SSL_kernel_crescentfullmoon_fs}, where the Gaussian kernel was used. For some parameter settings, the misclassification rates are slightly better, for other ones slightly worse.

\subsection{Kernel ridge regression}\label{sec::kernel_ridge_regression}
In this section we show that our approach can be applied to the problem of kernel ridge regression, which has a similar flavour to the problem from the previous section. We here illustrate that our method is very flexible since other than just Gaussian kernels can be used for the fast evaluation of matrix-vector products. The starting point is a simple linear regression problem via the minimization of 
\begin{equation}
\label{eq::krr}
\argmin_{\bu \in \R^d} \;
\frac{1}{2} \left\|\mathbf{f} - \mathbf{X}\bu\right\|_2^2
+\frac{\beta}{2} \norm{\bu}_{2}^{2},
\end{equation}
where $\mathbf{X}\in\R^{n\times d}$ is a design matrix %
holding training feature vectors $\mathbf{x}_j \in \R^d$ in its rows, i.e.
$\mathbf{X}^{T}=\left[\mathbf{x}_1,\ldots,\mathbf{x}_n\right]$,
and $\mathbf{f} \in \R^n$ is a given response vector.
The solution $\mathbf{u}$ to this problem can then be used in a linear model to predict a response for any new point $\mathbf{x} \in \R^d$ as $F(\mathbf{x}) = \mathbf{u}^T \mathbf{x}$.

\begin{sloppypar}
The well-known solution formula can be rearranged using the Sherman--Morrison--Woodbury formula to obtain
\begin{align*}
\mathbf{u}&=\left(\mathbf{X}^T\mathbf{X}+\beta\mathbf{I}_d\right)^{-1}\mathbf{X}^T\mathbf{f}\\
&=\left(\beta^{-1}\mathbf{I}_d-\beta^{-2}\mathbf{X}^{T}\left(\mathbf{I}_n+\beta^{-1}\mathbf{X}\mathbf{X}^T\right)^{-1}\mathbf{X}\right)\mathbf{X}^T\mathbf{f}\\
&=\mathbf{X}^T\left(\beta^{-1}\mathbf{I}_n-\beta^{-1}\left(\beta\mathbf{I}_n+\mathbf{X}\mathbf{X}^T\right)^{-1}\mathbf{X}\mathbf{X}^{T}\right)\mathbf{f}\\
&=\mathbf{X}^T\left(\beta\mathbf{I}_n+\mathbf{X}\mathbf{X}^T\right)^{-1}\left(\beta^{-1}\left(\beta\mathbf{I}_n+\mathbf{X}\mathbf{X}^T\right)-\beta^{-1}\mathbf{X}\mathbf{X}^{T}\right)\mathbf{f}\\
&=\mathbf{X}^T\left(\mathbf{X}\mathbf{X}^T+\beta\mathbf{I}_n\right)^{-1}\mathbf{f}.
\end{align*}
Using this formula, we can introduce the dual variable $\bm{\alpha} = \left(\mathbf{X}\mathbf{X}^T + \beta\mathbf{I}\right)^{-1} \mathbf{f}$
and rewrite the predicted response of a new point $\mathbf{x}$ as 
\[ F(\mathbf{x}) = \bu^T \mathbf{x} = \left(\mathbf{X}^T \bm{\alpha} \right)^T \mathbf{x} = \sum_{i=1}^n \bm{\alpha}_i \mathbf{x}_i^T \mathbf{x}. \]
An idea for increasing the flexibility of this method is replacing expressions $\mathbf{x}_i^T \mathbf{x}_j$ with $K(\mathbf{x}_i, \mathbf{x}_j)$ where $K : \R^d \times \R^d \to \R$ is an arbitrary kernel function \cite{robert2014machine}. This leads to replacing $\mathbf{X} \mathbf{X}^T$ with the Gram matrix $\mathbf{K}$ with entries
\begin{equation*}
\mathbf{K}_{ij}=K(\mathbf{x}_i,\mathbf{x}_j)\quad\forall\ i,j = 1,\ldots,n.
\end{equation*}
Consequently, the dual variable becomes $\bm{\alpha}=\left(\mathbf{K}+\beta\mathbf{I}_n\right)^{-1}\mathbf{f}$
and we obtain the kernel-based prediction function
\[ F(\mathbf{x}) = \sum_{i=1}^n \bm{\alpha}_i K(\mathbf{x}_i, \mathbf{x}). \]
For more details we refer to \cite{robert2014machine}. It is easily seen that the main effort of this algorithm goes into the computation of the coefficient vector $\bm{\alpha}=\left(\mathbf{K}+\beta\mathbf{I}_n\right)^{-1}\mathbf{f}.$ Note that this is were we again use the NFFT-based matrix vector products in combination with the preconditioned CG method as  the matrix $\mathbf{K}+\beta\mathbf{I}_n$ is positive definite and amenable to being treated using the NFFT for a variety of different kernel functions. In Figure~\ref{fig::KRR} we illustrate the results when kernel ridge regression is used with two different kernels, namely the Gaussian and the inverse multiquadric kernel.
\end{sloppypar}

\begin{figure}[htb!]
\centering{
\subfloat[Inverse multiquadric]{\label{fig::KKRI}
\includegraphics[trim=19cm 2.5cm 17cm 2cm,clip,width=7.5cm]{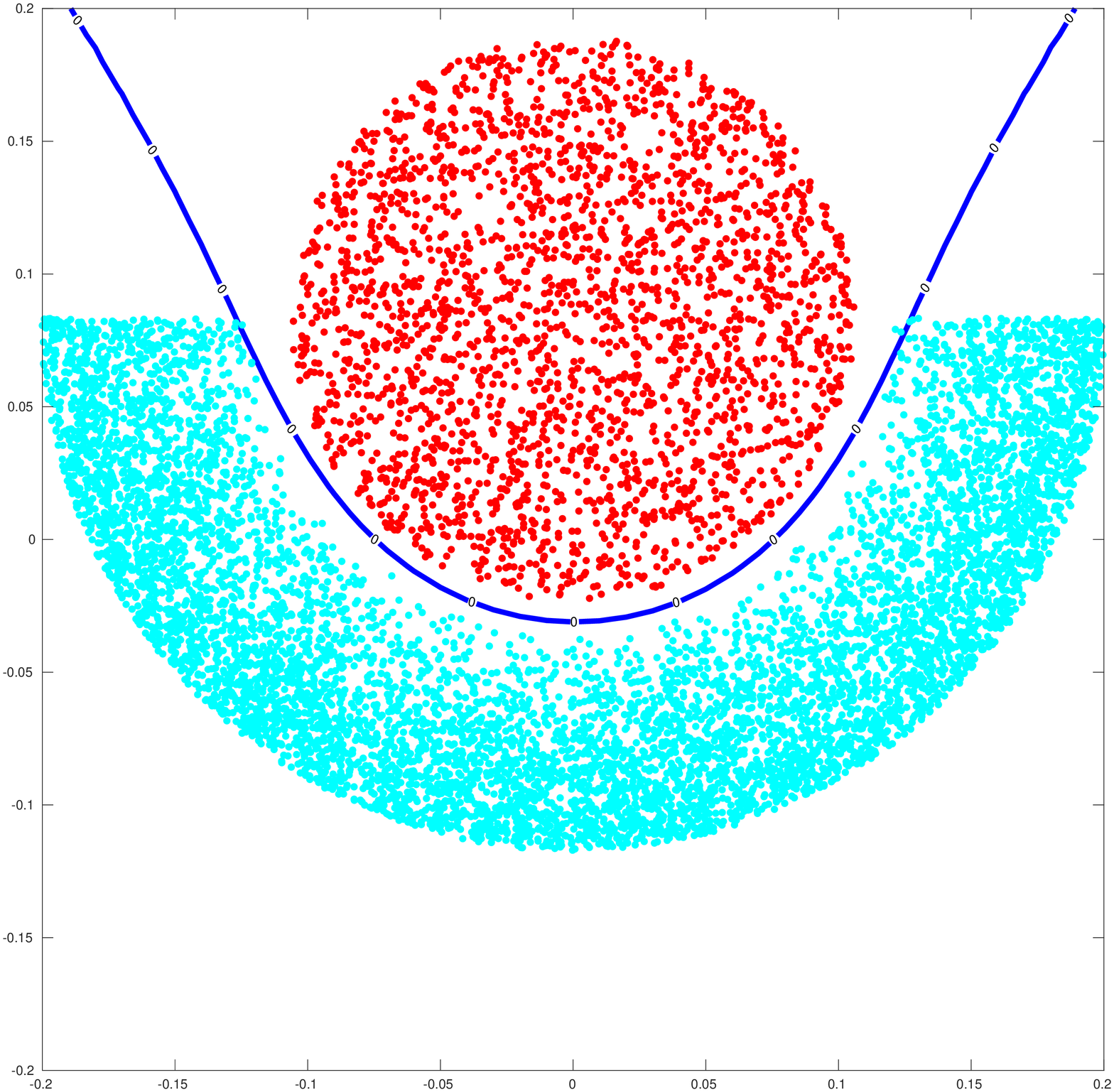}
}
\hfill
\subfloat[Gaussian]{\label{fig::KKRG}
\includegraphics[trim=18cm 2.5cm 18cm 2cm,clip,width=7.5cm]{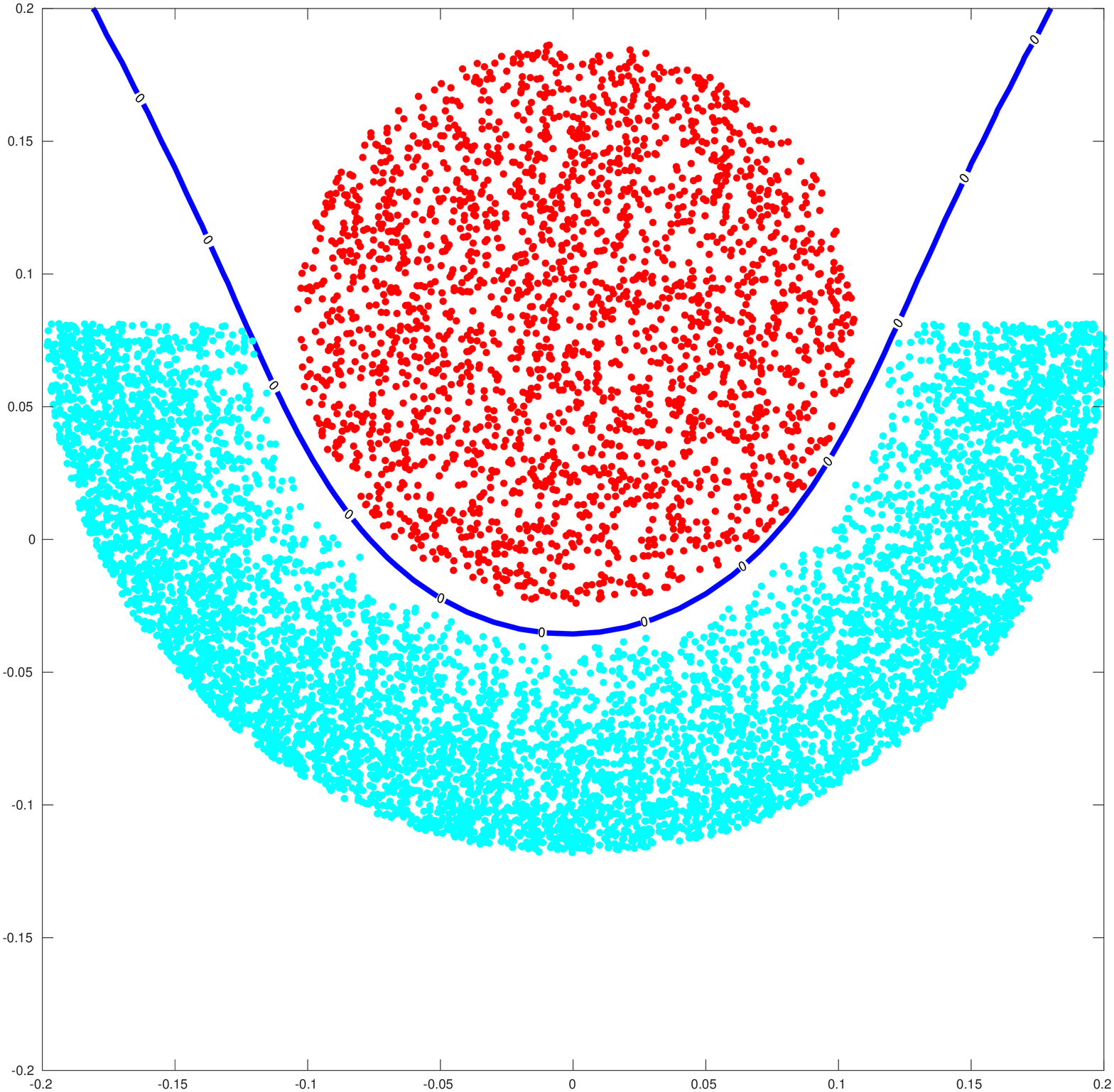}
}
}
\caption{Results of kernel ridge regression applied using an inverse multiquadric kernel (left) and a Gaussian kernel (right). The blue line indicates the decision boundary for the classification of new points.\label{fig::KRR}}%
\end{figure}

\section{Conclusion}
\label{sec::conclusion}

In this work, we have successfully applied the computational power of NFFT-based fast summation to core tools of data science. This was possible due to the nature of the fully connected graph Laplacian and the fact that many algorithms -- most notably the Lanczos method for eigenvalue computation -- only require matrix-vector products with the Laplacian matrix. By using Fourier coefficients to approximate the Gaussian kernel, we use Algorithm~\ref{alg::nfft_fastsum} to compute strong approximations of the matrix-vector product in $\mathcal{O}(n)$ complexity without storing or setting up the full matrix, as opposed to the full matrix's $\mathcal{O}(n^2)$ storage, setup, and application complexity.

For eigenvalue and eigenvector computations, we have discussed the current alternative method of choice in the Nystr\"om extension and developed a hybrid method that allows the basic Nystr\"om idea to benefit from NFFT-based fast matrix-vector products. In our numerical experiments, we found that the Nystr\"om-Gaussian-NFFT method achieved much better eigenvalue accuracy than the traditional Nystr\"om extension even for a significantly smaller parameter $L$, but was in turn outperformed by the NFFT-based Lanczos method.

In strongly eigenvector-dependent applications like in Section~\ref{sec::res::app::pf}, the higher accuracy of the NFFT-based Lanczos method directly leads to better classification results.
In some other applications, however, it is hard to predict if better eigenvector accuracy distinctly improves the results. For instance in Section~\ref{sec::res::app::clustering}, the traditional Nystr\"om extension still achieved good image clusterings on average with small parameter $L$ despite its rather inaccurate eigenvectors. Here, the NFFT-based Lanczos method still has very good selling points in its greatly improved runtime as well as its consistency, while the traditional Nystr\"om tends to ``fail'' in some test runs.

\end{document}